\documentclass[runningheads]{llncs}

\usepackage{eccv}
\usepackage{eccvabbrv}

\usepackage{graphicx}
\usepackage{booktabs}
\usepackage[accsupp]{axessibility}
\usepackage{hyperref}
\usepackage{orcidlink}

\usepackage{subcaption}
\usepackage{multirow}
\usepackage{array}
\usepackage{pifont}
\usepackage{tcolorbox}
\usepackage{makecell}
\usepackage[table]{xcolor}

\usepackage{amsmath,amsfonts,bm}

\def\eqref#1{equation~\ref{#1}}

\def\1{\bm{1}}

\def\rvu{{\mathbf{i}}}

\def\rvm{{\mathbf{m}}}

\def\rvu{{\mathbf{u}}}
\def\rvv{{\mathbf{v}}}

\def\rvx{{\mathbf{x}}}

\def\rvz{{\mathbf{z}}}

\def\mI{{\bm{I}}}

\DeclareMathAlphabet{\mathsfit}{\encodingdefault}{\sfdefault}{m}{sl}
\SetMathAlphabet{\mathsfit}{bold}{\encodingdefault}{\sfdefault}{bx}{n}

\newcommand{\E}{\mathbb{E}}
\newcommand{\Ls}{\mathcal{L}}
\newcommand{\R}{\mathbb{R}}

\newcommand{\revised}[1]{\textcolor{black}{#1}}
\newcommand{\equalcontrib}{\textsuperscript{*}}
\newcommand{\projectlead}{\textsuperscript{\textdagger}}
\newcommand{\correspondingauthor}{\textsuperscript{\textdaggerdbl}}

\begin{document}

\title{Delving into Latent Spectral Biasing of Video VAEs for Superior Diffusability} 

\titlerunning{SSVAE: Latent Spectral Biasing for Video VAEs}

\author{Shizhan Liu\equalcontrib\inst{1}\orcidlink{0009-0000-8353-8727} \and
Xinran Deng\equalcontrib\inst{1,3}\orcidlink{0009-0009-2619-7554} \and
Zhuoyi Yang\projectlead\inst{2} \and
Jiayan Teng\inst{2}\orcidlink{0000-0003-1346-734X} \and \\
Xiaotao Gu\inst{1} \and
Jie Tang\correspondingauthor\inst{2}
}

\authorrunning{S.~Liu et al.}

\institute{Zhipu AI, Beijing, China \and
Tsinghua University, Beijing, China \and
University of Chinese Academy of Sciences, Beijing, China}

\maketitle
\begingroup
\renewcommand{\thefootnote}{\fnsymbol{footnote}}
\footnotetext[1]{Equal contribution. \projectlead Project leader. \correspondingauthor Corresponding author.}
\endgroup

\begin{abstract}
  Latent diffusion models pair VAEs with diffusion backbones, and the structure of VAE latents strongly influences the difficulty of diffusion training. However, existing video VAEs typically focus on reconstruction fidelity, overlooking latent structure. We present a statistical analysis of video VAE latent spaces and identify two spectral properties essential for diffusion training: a channel-wise eigenspectrum dominated by a few modes, and a spatio-temporal frequency spectrum biased toward low frequencies. To induce these properties, we propose two lightweight, backbone-agnostic regularizers: Latent Masked Reconstruction and Local Correlation Regularization. Experiments show that our Spectral-Structured VAE (SSVAE) achieves a $3\times$ speedup in text-to-video generation convergence and a 10\% gain in video reward, outperforming strong open-source VAEs. Code is available at: \href{https://github.com/zai-org/SSVAE}{https://github.com/zai-org/SSVAE}.
\end{abstract}

\section{Introduction}
\label{sec:intro}
Latent video diffusion models~\cite{wan, cogvideox, hunyuan-video} have recently advanced text-to-video (T2V) generation by coupling a 3D VAE-based tokenizer with a diffusion backbone. A growing body of evidence indicates that the VAE strongly shapes downstream diffusion training dynamics~\cite{SELoss,VA-VAE,MAEtok}.

Unfortunately, most existing video VAEs, including those adopted in SOTA T2V models, pursue better temporal compression and high-fidelity reconstruction through architectural design~\cite{step-video,IV-VAE,WF-VAE} and the optimization of reconstruction-based objectives~\cite{wan,hunyuan-video,cogvideox} (e.g., MSE, adversarial, LPIPS). The latent structure, however, receives limited attention in their optimization. This objective-target mismatch leads to a well-known phenomenon: stronger reconstruction fidelity does not necessarily translate into better generative utility~\cite{VA-VAE}.

\begin{figure}[t]
  \centering
   \includegraphics[width=\linewidth]{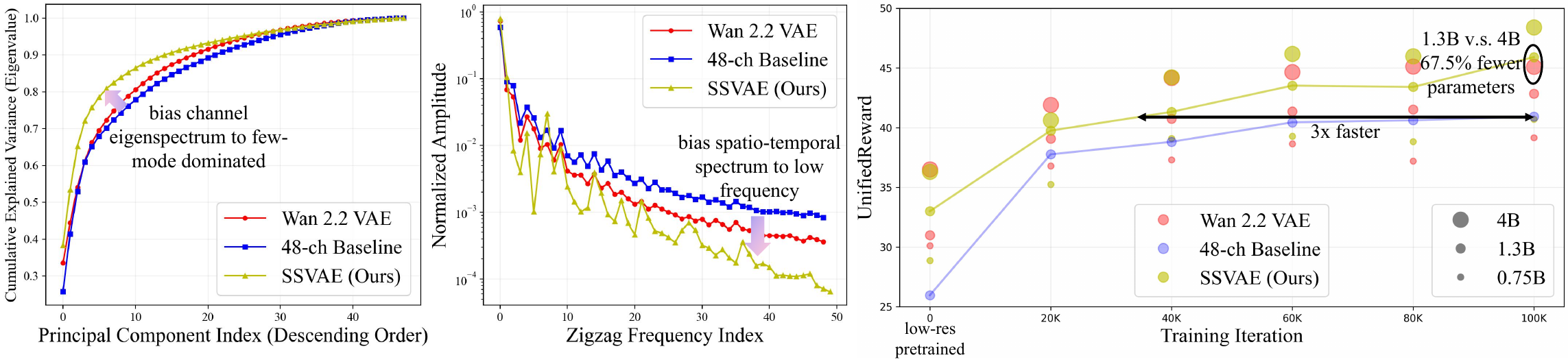}
   \caption{We identify that both a low-frequency biased spatio-temporal frequency spectrum and a few-mode biased channel eigenspectrum facilitate diffusion training. By inducing low-frequency bias, few-mode bias and enhancing decoder robustness, our SSVAE achieves a $3\times$ convergence speedup over the baseline on $17\times 512\times 512$ generation, using prompts from VBench.}
   \label{fig:teaser}
\end{figure}

Prior work on image VAEs has proposed regularizers and analyses from various perspectives, yet the central question remains open: \textbf{What properties should a latent space possess to facilitate downstream generative training?} Existing analyses often rely on indirect proxy metrics, such as ImageNet linear probing accuracy~\cite{MAEtok} or similarity to DINOv2 features~\cite{dinov2,REPA-E}, yielding only an indirect characterization of VAE latents. Optimizing VAEs for these metrics often requires specialized designs. Other representation-level explanations, such as latent clustering~\cite{VA-VAE}, provide useful intuition but lack measurable quantities that connect the intuition to generation quality. Moreover, existing analyses are usually studied in isolation, leaving their overlap and complementarity unclear, which makes it difficult to distill a principled set of `properties beneficial for generation.'

In this work, we seek to identify properties of video VAE latent spaces that facilitate diffusion training, which hereafter referred to as \emph{diffusability}, through statistical analysis of VAE latents. As direct characterizations of the latent space, statistical measures can not only offer actionable targets for designing regularizers, but also shed light on the mechanisms behind existing proxy-metric approaches. Moreover, the overlap and complementarity among various statistical analyses open up possibilities for combining different regularizations. Specifically, for video VAE latents with temporal, height, width, and channel dimensions, we analyze two complementary spectra: the channel-wise eigenspectrum and the spatio-temporal frequency spectrum. \textbf{Our key finding is that biased, rather than uniform, spectra lead to improved diffusability.}

For the eigenspectrum, we show that a few-mode-biased channel-wise eigenspectrum enhances diffusability. In other words, the latent code has a low effective rank and can be closely approximated by a linear combination of only a few basis vectors. By analyzing the learning dynamics of diffusion models along each eigenvector, we reveal the mechanism by which few-mode bias accelerates convergence. Finally, we propose Latent Masked Reconstruction (LMR), which simultaneously promotes few-mode bias and improves decoder robustness, resulting in substantially faster convergence of diffusion models.

We then turn to the complementary spatio-temporal frequency spectrum. SER~\cite{SELoss} has shown that spatial frequency spectra biased toward low-frequency components improve latent-space diffusability, and we extend this analysis to the spatio-temporal domain, demonstrating that two common strategies in image VAEs, scale-equivariant regularization~\cite{SELoss,EQVAE} and foundation-model alignment~\cite{VA-VAE,REPA-E}, both promote low-frequency bias. However, while they introduce non-negligible computational overhead, they do not adequately address the temporal dimension in video latents. From a statistical perspective, we find that the proportion of low-frequency energy is positively correlated with the local spatio-temporal correlation of the latents. This motivates our development of Local Correlation Regularization (LCR), a computationally efficient regularizer that explicitly enhances local correlation, thereby increasing low-frequency energy.

Based on LMR and LCR, we train a \textbf{Spectral-Structured VAE (SSVAE)} that significantly improves diffusability, as shown in Fig.~\ref{fig:teaser}. We summarize our contributions as the following:
\begin{itemize}
\item We present the first detailed statistical analysis of the latent space in video VAEs, revealing two spectral properties that are crucial for generative training: a channel eigenspectrum dominated by few modes, and a spatio-temporal frequency spectrum biased toward low frequencies. Our findings provide new insights into the underlying mechanisms of existing methods and offer actionable design targets for regularization.
\item We propose Latent Masked Reconstruction (LMR) to promote few-mode bias and decoder robustness to noise, and Local Correlation Regularization (LCR) to bias latents toward low frequencies. Both are lightweight, backbone-agnostic, and easy to implement.
\item Extensive experiments across diffusion backbones and parameter scales show that our training recipe substantially accelerates convergence by $3\times$ and improves video reward by $10\%$, consistently outperforming strong open-source VAE baselines for text-to-video.
\end{itemize}

\section{Related Work}
\label{sec:related_works}
\subsection{Video VAE}
Existing video VAEs primarily focus on improving spatio-temporal compression and video reconstruction. Early efforts extend image VAEs to video by leveraging temporal redundancy for compression~\cite{OD-VAE,CV-VAE}. Subsequent work explores keyframe-based temporal compression and grouped causal convolutions to mitigate inter-frame imbalance~\cite{IV-VAE}, and multi-level wavelet transforms to enhance temporal coherence with efficient architectures~\cite{WF-VAE}. In large-scale text-to-video models, 3D CNN-based VAEs have become standard tokenizers, as adopted by CogVideoX~\cite{cogvideox}, HunyuanVideo~\cite{hunyuan-video}, and Wan2.1~\cite{wan}.

Despite solid progress on compression and reconstruction fidelity, most video VAEs are still trained with reconstruction-centric objectives (e.g., L1/L2, LPIPS, adversarial losses), with limited attention to the latent space properties that affect downstream diffusion training. Consequently, they are ill-equipped to address the well-known trade-off between reconstruction and generation~\cite{VA-VAE}. Our work focuses on shaping video VAE latents for generation.

\subsection{Generation-Oriented Image Tokenizer}
Numerous analyses and regularizers have been proposed for generation-oriented image tokenizers. Some works design and optimize proxy metrics linked to VAE latents, such as MAEtok~\cite{MAEtok}, which correlates ImageNet linear probing accuracy with generative quality, and REPA-E~\cite{REPA-E}, which uses diffusion backbone–DINOv2 feature similarity. However, the relationships between proxy metrics and VAE latents are often complex and indirect, making it difficult to identify underlying principles. Recent work also explores how to make reconstruction FID predictive of downstream diffusion generation FID~\cite{xu2026making}, offering an output-level, evaluation-oriented perspective on tokenizer quality rather than a direct VAE optimization objective. Other works visualize the latent space and offer qualitative interpretations. For example, VA-VAE~\cite{VA-VAE} claims latent clustering is detrimental, and DC-AE-1.5~\cite{DC-AE-1.5} emphasizes preserving object structure. Yet their methods only indirectly optimize the objectives motivated by these analyses, which limits the extent to which analytical improvements translate into better generation. There are also works that perform frequency-based analyses, such as SER~\cite{SELoss}, which suppress high-frequency components indirectly by enforcing scale equivariance, but at notable computational cost.

Despite these diverse perspectives, most analyses are conducted in isolation. As a result, a principled set of generative-beneficial properties is still lacking.

\section{Channel Eigenspectrum Shaping}
We first analyze the channel eigenspectrum of video VAEs. We observe that VAEs exhibiting Few-Mode Bias (FMB) and enhanced decoder robustness possess better diffusability. Next, we detail the impact of FMB on diffusion training in Sec.~\ref{sec:FMB_impact}, and analyze its convergence accelerating mechanism from a cross-correlation view in Sec.~\ref{sec:cross_correlation}. Finally, Sec.~\ref{sec:latent_masked_recons} introduces the Latent Masked Reconstruction (LMR) to simultaneously promotes FMB and decoder robustness.

\subsection{Impact of FMB on Diffusion Training}
\label{sec:FMB_impact}
We begin by analyzing the channel eigenspectrum of VAEs with different channel counts, motivated by the well-known observation that downstream generative training converges slower on VAEs with more channels. For a $C$-channel VAE, let $\rvu^0 \in \mathbb{R}^{1\times C}$ denote a latent vector sampled from the normalized latent $\tilde{\rvz}$, where the superscript $0$ denotes the clean latent at diffusion timestep $0$. The channel-wise autocorrelation matrix $\Sigma_{\rvu \rvu} \in \mathbb{R}^{C \times C}$ is then computed as $\Sigma_{\rvu \rvu} = \mathbb{E}[(\rvu^0)^\top \rvu^0]$, where the expectation is taken over the sample and spatio-temporal dimensions. Subsequently, we apply principal component analysis (PCA) to $\Sigma_{\rvu \rvu}$ to analyze its eigenvalue spectrum. We train 48-, 64-, and 128-channel VAEs using standard losses (L1, KL, LPIPS, and adversarial), and compare the cumulative explained variance curves in Fig.~\ref{fig:pca_comparison}. Here, `explained variance' corresponds to eigenvalues sorted in descending order. The differences are substantial: \textbf{higher-channel VAEs tend to distribute eigenvalues more evenly across eigenvectors, whereas lower-channel VAEs concentrate eigenvalues in a few dominant eigenvectors.} We term this phenomenon `Few-Mode Bias' (FMB). Here, `mode' refers to the eigenvector. A highly few-mode biased latent space implies a low effective rank across channels and can be well approximated by a linear combination of a few eigenvectors.

We further investigate the impact of FMB on generation training in VAEs with equal channel counts. We train a baseline 48-channel VAE and a covariance-regularized variant by penalizing the total eigenvalues outside the top three principal components. Our findings are as follows: (i) \textbf{The FVD trend is consistent with the eigenvalue distribution.} The covariance-penalized VAE effectively induces a few-mode bias and achieves a lower FVD, as shown in Fig.~\ref{fig:pca_comparison} and Fig.~\ref{fig:fvd_comparison}. (ii) \textbf{Few-mode biased models achieve higher video rewards}, as depicted in Fig.~\ref{fig:video_reward_ch48}. Based on these observations, we hypothesize that a few-mode biased latent space benefits diffusion training and proceed to analyze the underlying reasons.

\begin{figure*}[t]
  \centering
  \begin{subfigure}[b]{0.19\linewidth}
    \includegraphics[width=\linewidth]{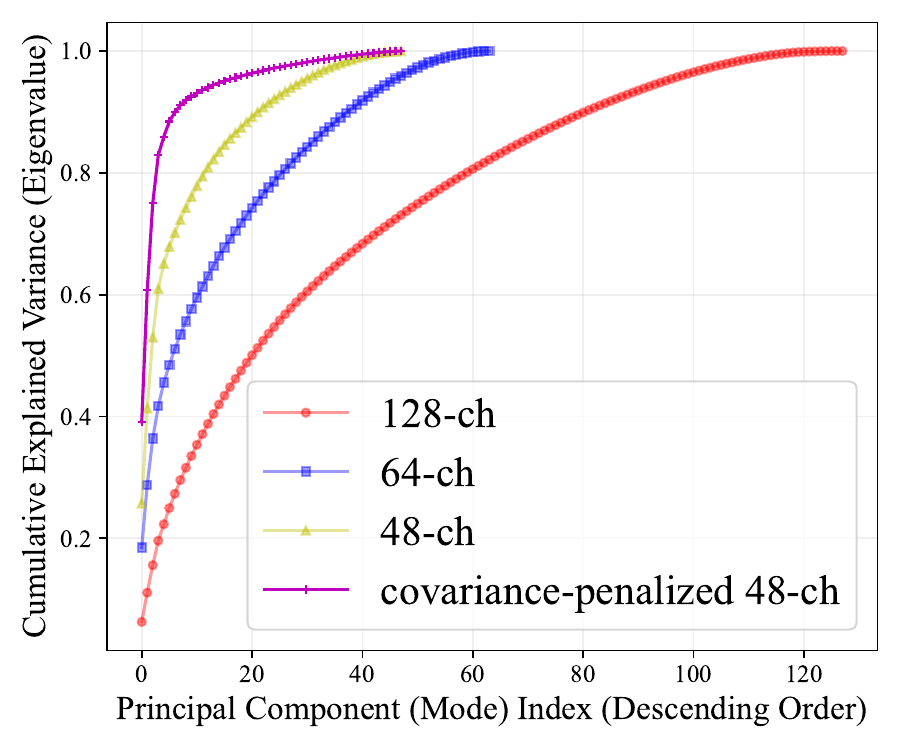}
    \caption{Mode strength comparison of VAE latent covariance.}
    \label{fig:pca_comparison}
  \end{subfigure}\hfill
  \begin{subfigure}[b]{0.19\linewidth}
    \includegraphics[width=\linewidth]{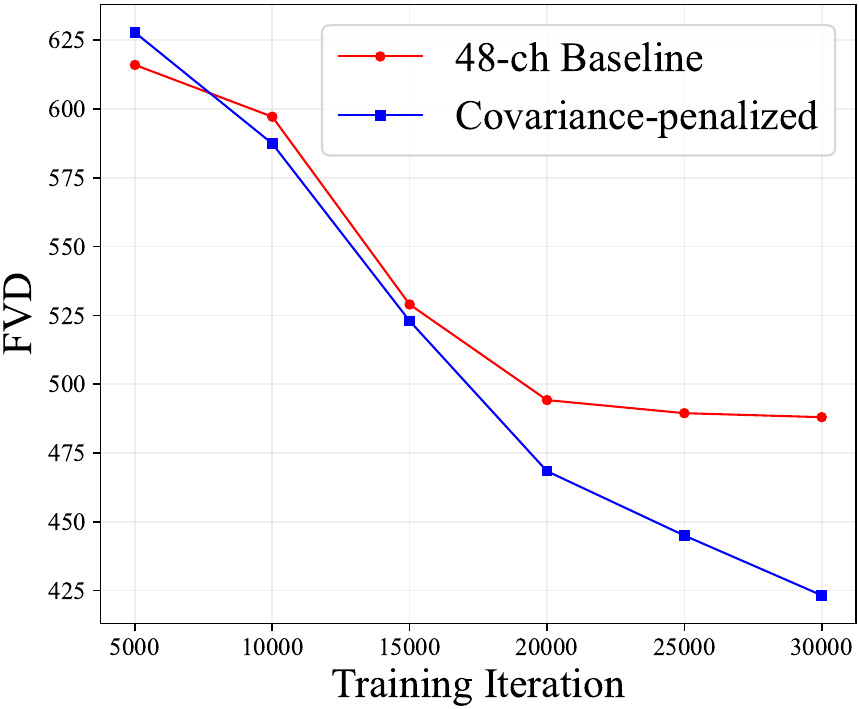}
    \caption{Generation quality comparison by FVD.}
    \label{fig:fvd_comparison}
  \end{subfigure}\hfill
  \begin{subfigure}[b]{0.19\linewidth}
    \includegraphics[width=\linewidth]{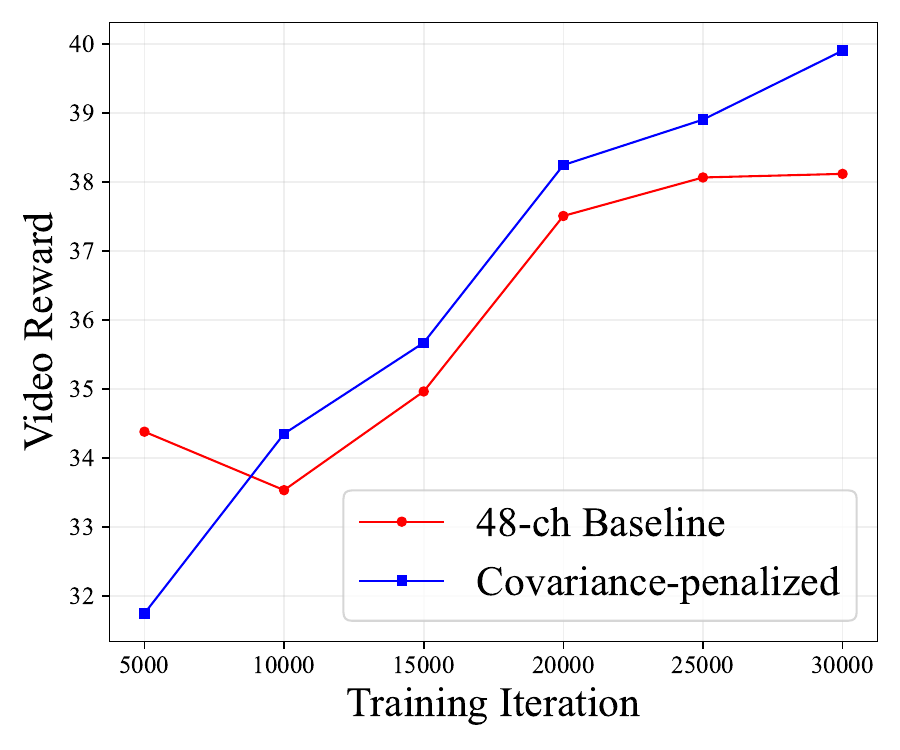}
    \caption{Generation quality comparison by UnifiedReward.}
    \label{fig:video_reward_ch48}
  \end{subfigure}\hfill
  \begin{subfigure}[b]{0.19\linewidth}
    \includegraphics[width=\linewidth]{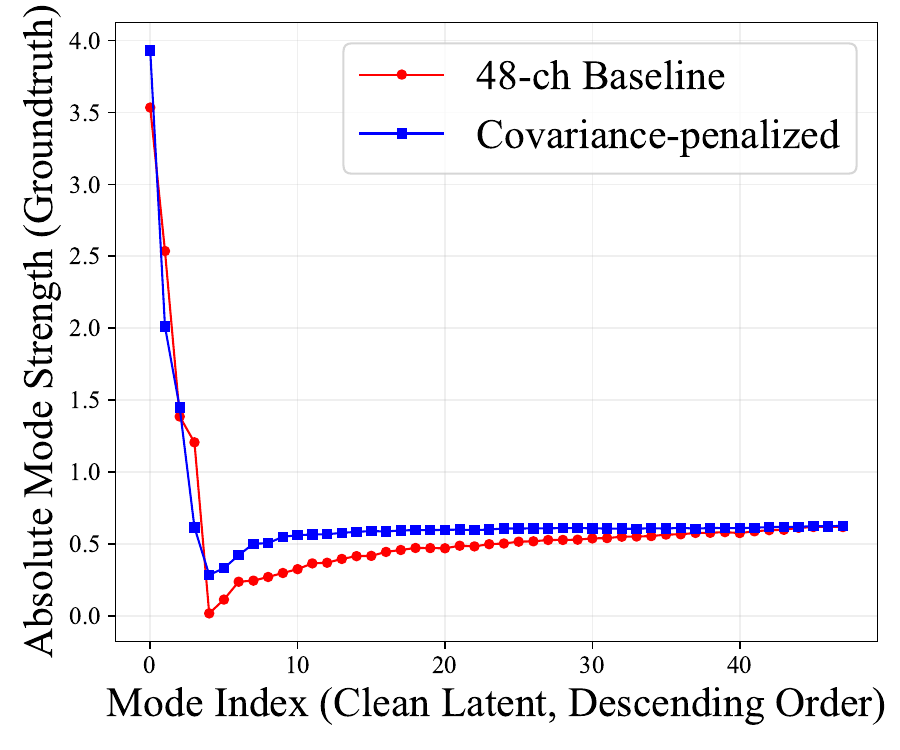}
    \caption{Mode strength comparison of cross correlation.}
    \label{fig:cross_corres_comparison}
  \end{subfigure}\hfill
  \begin{subfigure}[b]{0.19\linewidth}
    \includegraphics[width=\linewidth]{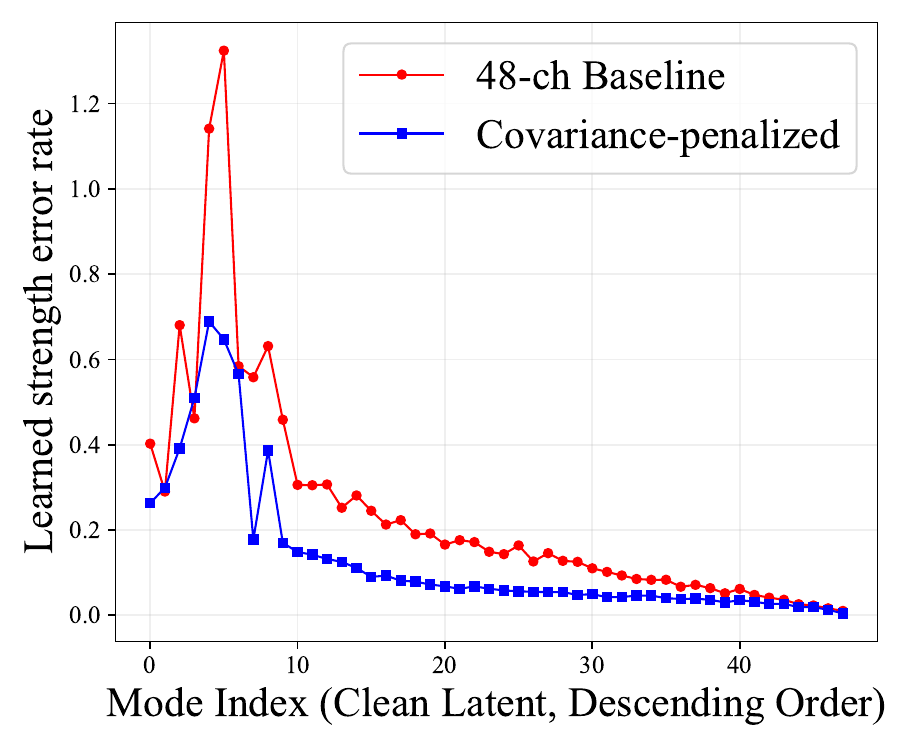}
    \caption{Mode-wise learned error in cross correlation.}
    \label{fig:relative_mode_error}
  \end{subfigure}
  \caption{Comparative analysis of the VAE latent channel covariance matrix and the diffusion output–input cross-correlation matrix. A few-mode-biased latent space is associated with a lower FVD, higher generation quality, and faster convergence.}
  \label{fig:FMB_impact}
\end{figure*}
\subsection{A Cross-Correlation View of FMB}
\label{sec:cross_correlation}
To uncover why a few-mode biased latent space accelerates diffusion training, we analyze the learning dynamics using output-input cross-correlation, which is employed in the convergence analysis framework of deep linear networks~\cite{saxe2013exact} \revised{and has recently been shown to be effective for non-linear diffusion models~\cite{bonnaire2025diffusion}.} Output-input cross-correlation captures the statistical dependencies that the model needs to learn. By analyzing its eigenspectrum, we can quantitatively evaluate training progress for each `correlation mode', namely the eigenvectors of the groundtruth cross-correlation matrix. A mode is considered well learned if the strength of the learned cross-correlation along its corresponding eigenvector approaches the groundtruth eigenvalue.

Specifically, let $\Sigma_{\rvv \rvu}(t) \in \mathbb{R}^{C \times C}$ denote the channel-wise output-input cross-correlation matrix of the diffusion backbone at timestep $t$, and let $\lambda_l$ denote the $l$-th largest eigenvalue of the autocorrelation $\Sigma_{\rvu\rvu}$ that we analyzed in Sec.~\ref{sec:FMB_impact}. Then we have the following theorem under flow-matching and velocity prediction.
\begin{theorem}
\label{thm:eigen_value_relation}
$\Sigma_{\rvv\rvu}(t)$ has the same eigenvectors as $\Sigma_{\rvu\rvu}$, and its eigenvalue on the $l$-th eigenvector of $\Sigma_{\rvu\rvu}$ is given by:
\begin{equation}
s_l(t) = t - (1-t)\lambda_l.
\label{eq:strength_correlation}
\end{equation}
Please refer to Sec.~\ref{sec:proof} for the proof and scope of Theorem~\ref{thm:eigen_value_relation}.
\end{theorem}

This eigenspectrum relationship bridges the structure of VAE latents and the diffusion model’s training dynamics. Since $\Sigma_{\rvv\rvu}(t)$ and $\Sigma_{\rvu\rvu}$ share eigenvectors, convergence analysis of $\Sigma_{\rvv\rvu}(t)$ allows us to assess the diffusion backbone’s learning progress in each mode of the clean latent. Specifically, to evaluate overall model convergence across all timesteps, we employ Monte Carlo simulation to estimate the expected eigenvalues, $\bar{s}_l$, of the ground-truth cross-correlation matrix, under the widely adopted Logit-normal timestep sampling strategy~\cite{SD3} (see Fig.~\ref{fig:cross_corres_comparison}). By comparing these ground-truth mode strengths $\bar{s}_l$ to those learned by the model, we quantitatively evaluate convergence rates for different correlation modes (see Fig.~\ref{fig:relative_mode_error}). This analysis leads to two key findings: (i) \textbf{Few-mode biasing can result in all absolute mode strengths, $|\bar{s}_l|$, exceeding those of the baseline}. This phenomenon arises from the non-monotonic relationship between $|\bar{s}_l|$ and $\lambda_l$ (see Eq.~\ref{eq:strength_correlation}): as $\lambda_l$ decreases, the corresponding $|\bar{s}_l|$ first decreases and then increases. (ii) \textbf{Modes with larger absolute strengths tend to converge faster}, extending the insights of~\cite{saxe2013exact} to diffusion training. Fig.~\ref{fig:relative_mode_error} shows that the covariance-penalized version has lower errors in the learned strength for each mode than the baseline. 

Taken together, these results suggest that \textbf{few-mode biased latent spaces can accelerate diffusion model convergence by amplifying the absolute mode strengths in the output-input cross-correlation matrix.}

\subsection{Latent Masked Reconstruction}
\label{sec:latent_masked_recons}
To reliably promote a few-mode biased VAE latent space, we introduce Latent Masked Reconstruction (LMR) during training. Unlike the covariance penalization method in Sec.~\ref{sec:FMB_impact}, \textbf{LMR avoids explicit eigenvalue decomposition, reducing computational overhead and improving numerical stability}. The intuition behind LMR is that, to minimize the influence of random masking on reconstruction, the VAE encoder must concentrate essential information into a few modes, thus promoting FMB.

Specifically, LMR randomly replaces latent vectors with mask tokens across spatio-temporal dimensions, feeding the masked latents to the decoder for reconstruction, as shown in Fig.~\ref{fig:LMR}(a). The masking ratio is randomly selected from $\{0, 0.25, 0.5, 0.75\}$, where $0$ indicates no masking. Let $\mathcal{M} \in \mathbb{R}^{T \times H \times W \times 1}$ denote the spatio-temporal mask and $\rvm \in \mathbb{R}^{1 \times 1 \times 1 \times C}$ the mask token. With the broadcast mechanism, this process can be formulated as:
\begin{equation}
\Ls_{LMR} = \mathcal{D}( \rvx, Dec(\rvz \odot \mathcal{M} + (1-\mathcal{M}) \odot \rvm)),
\label{eq:latent_mask_recons}
\end{equation}
here $\rvx$ denotes the video, $\rvz$ its corresponding latent code, $Dec(\cdot)$ the VAE Decoder, and $\odot$ the Hadamard product. \revised{The video reconstructed via LMR is then used to compute the distance metrics, including the L1, LPIPS, and GAN losses.}

Beyond facilitating FMB, an equally important role of LMR is to enhance decoder robustness. Video VAEs in T2V models often use a very small KL loss weight~\cite{wan} to produce latents easier to model. While this sufficiently constrains the posterior mean, it can cause posterior variance collapse. Wan 2.1 VAE, for example, uses a KL loss weight of $1\mathrm{e}{-6}$, yielding a posterior variance as low as $1\mathrm{e}{-7}$. In contrast, the typical reconstruction error in diffusion models is at the $10^{-1}$ magnitude, which makes it difficult for the VAE decoder to generate reasonable videos from highly noisy diffused samples. The masking operation in LMR, which replaces latent features with mask tokens, can also be interpreted as a kind of noise injection. Consequently, LMR enhances both the mode distribution and decoder robustness. We further discuss the influence of decoder robustness in Sec.~\ref{sec:incre_comp_analysis} and Tab.~\ref{tab:component_analysis}.

\revised{To further justify our design choices, we compare LMR with other regularizations. Unlike the image tokenizer MAEtok~\cite{MAEtok},  which performs random masking in pixel space, LMR masks latents directly and is thus compatible with latent regularizers such as LCR. Compared to the Channel-wise Progressive Reconstruction (CPR) in DC-AE-1.5~\cite{DC-AE-1.5}, which randomly retains the first $C'$ channels for reconstruction, we find that random spatial-temporal masking more effectively encourages few-mode bias. Reproducing CPR under the same masking ratios and probabilities shows that it mainly amplifies the primary eigenvalue but distributes the remaining ones more uniformly than the baseline (see Fig.~\ref{fig:LMR}(b)), leading to slower convergence on most modes (Fig.~\ref{fig:LMR}(c)). This may explain why CPR requires \textbf{modifying the diffusion training} for progressive channel-wise denoising, whereas LMR remains effective as a plug-and-play regularization without altering the subsequent diffusion training.}

\begin{figure}[t]
    \centering
    \includegraphics[width=\linewidth]{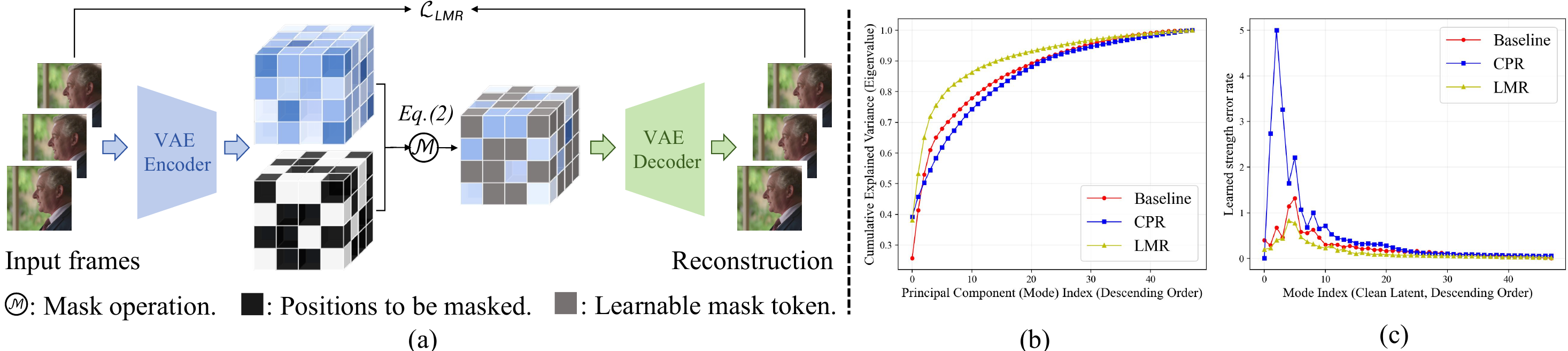}
    \caption{(a) LMR introduces a few-mode bias by reconstructing videos using spatio-temporally masked latents. We omit the channel dimension for simplicity. (b) Cumulative explained variance comparison between CPR and LMR. (c) Eigenspectrum shaping effects comparison of LMR and CPR.}
    \label{fig:LMR}
\end{figure}

\section{Spatio-Temporal Frequency Spectrum Shaping}
Complementary to the channel eigenspectrum, we further analyze the spatio-temporal frequency spectrum. Existing work~\cite{SELoss} has shown that biasing the spatial frequency spectrum toward low frequencies improves diffusability, and we extend this analysis to the spatio-temporal frequency spectrum, identify commonalities and limitations in existing methods, and propose Local Correlation Regularization, a regularizer founded on the statistical link between low-frequency energy and local correlation.

\subsection{Preliminary}
To analyze the frequency characteristics of latents, a widely adopted approach is first apply a frequency transform to map the latents into the frequency domain, and then compute the energy at each frequency to obtain the power spectral density (PSD). Existing work~\cite{SELoss} employs a 2D discrete cosine transform (DCT) to analyze the spatial frequency spectrum, averaging across the temporal and channel dimensions. Following this procedure, we instead adopt a 3D DCT to analyze the spatio-temporal frequency spectrum, averaging only across the channel dimension. For visualization, we linearize the 3D frequency grid via zigzag ordering, bin consecutive frequencies, sum the energy within each bin, and normalize the spectral energy, as depicted in Fig.~\ref{fig:LCR}(a).

SER~\cite{SELoss} observes that \textbf{a spatial spectrum biased toward low frequencies (or, equivalently, the suppression of high-frequency components) correlates with improved diffusion training.} Intuitively, high-SNR low-frequency components facilitate the recovery of low-SNR high-frequency details during denoising, simplifying optimization. We find that this low-frequency biasing principle likewise applies to the spatio-temporal frequency spectrum.

\subsection{From PSD to Local Correlation}
Existing works encourage low-frequency bias by enforcing spatial scale equivariance in latents~\cite{SELoss,EQVAE}. Notably, we observe that foundation-model alignment~\cite{VA-VAE} similarly increases low-frequency energy, although it is not claimed. This may partly explain its effectiveness. However, neither approach addresses temporal frequencies, failing to adequately bias the spatiotemporal spectrum toward low frequencies, as shown in Fig.~\ref{fig:LCR}(a). Moreover, both incur non-trivial computational overhead. Thus, we seek a new regularizer that operates directly on the spatio-temporal spectrum while remaining computationally efficient.

Our key insight is that \textbf{low-frequency energy is largely governed by the similarity of latent vectors at neighboring spatio-temporal positions.} When adjacent positions have highly similar channel vectors, the latent becomes smoother, concentrating more power in low frequencies. This follows the Wiener–Khinchin theorem~\cite{wiener-khincin}, which states that the power spectral density (PSD) and the autocorrelation function form a Fourier pair. Consequently, boosting small-lag correlations increases low-frequency power. Guided by this observation, we propose a simple yet effective Local Correlation Regularization (LCR), which explicitly enhances similarities within small spatio-temporal patches, is computationally efficient, and naturally addresses the temporal dimension. A discussion of the relationships between the autocorrelation function, low-frequency energy, and local correlation, along with a complexity analysis of LCR, is presented in Sec.~\ref{sec:autocorrelation}.

\begin{figure}[t]
  \centering
    \centering
    \includegraphics[width=\linewidth]{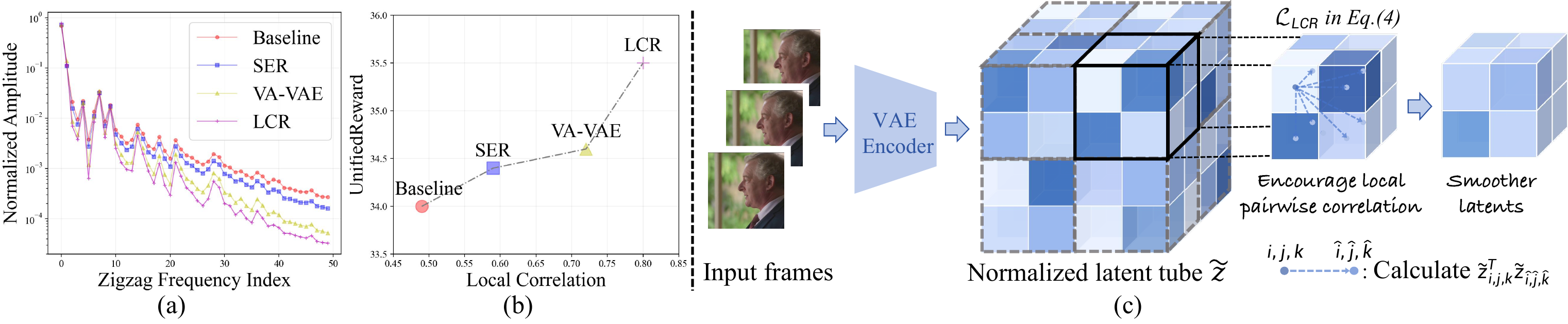}
    \caption{(a) Comparison of PSD curves. Low zigzag frequency indexes correspond to low frequencies. LCR exhibits the strongest but controlled low-frequency bias. (b) Comparison of UnifiedReward over local correlation value. Steeper PSDs correspond to larger local correlation, and result in better video generation quality. (c) LCR introduces a low-frequency bias by promoting pairwise correlations within each spatio-temporal local patch in the normalized latents. We omit the channel dimension for simplicity.}
    \label{fig:LCR}
\end{figure}

Concretely, given a per-channel standardized latent, we partition it spatio-temporally into non-overlapping small patches and measure the average pairwise correlation within each patch. We denote the standardized latent as $\tilde{\rvz}\in \R^{T\times H\times W\times C}$, where $T, H, W, C$ denote temporal length, height, width, and number of channels. During training, the mean and variance of each channel are computed over $(B, T, H, W)$ to match the diffusion backbone’s preprocessing, where $B$ is the batch size. Using $(i,j,k)$ to index time, height and width, and letting $p$ denote the set of spatio-temporal coordinates within a patch, the average local correlation in patch $p$ is
\begin{align}
\tilde R(p) \;=\; & \E_{\substack{(i,j,k),(\hat{i}, \hat{j}, \hat{k})\in p,\\ (i,j,k)\neq(\hat{i}, \hat{j}, \hat{k})}} \big[\, \tilde{\rvz}^\top_{i,j,k}\, \tilde{\rvz}_{\hat{i},\hat{j},\hat{k}} \,\big].
\label{eq:auto_correlation_local}
\end{align}
We define Local Correlation Regularization (LCR) as a hinge loss on the patch-averaged $\tilde R(p)$:
\begin{equation}
\Ls_{LCR} \;=\; \mathrm{ReLU}\!\left(\alpha - \E_p \big[ \tilde R(p) \big]\right),
\label{eq:LCR}
\end{equation}
where $\alpha$ is a saturation threshold that prevents over-smoothing; together with reconstruction losses, it keeps the bias \textbf{controlled}, as excessive bias can hurt generation (see Sec.~\ref{sec:lcr_hyperparameters}). In practice, we set the patch size to 2, dividing spatial patches in the first frame of the latent, and spatio-temporal patches in the remaining frames. \revised{We apply LCR at the \textbf{batch level} rather than per-sample, which preserves the flexibility of local correlations, accommodating videos with high dynamics and complex details}. We further observe that replacing the dot product in Eq.~\ref{eq:auto_correlation_local} with cosine similarity preserves effective autocorrelation optimization while better balancing local correlations between the first and subsequent frames. Accordingly, this choice yields the Pearson correlation coefficient within each local patch and is adopted in our implementation. LCR is also illustrated in Fig.~\ref{fig:LCR}(c).

In Fig.~\ref{fig:LCR}(b), we observe that for 64-channel VAEs, \textbf{local correlation both reflects the steepness of the PSD and correlates positively with the reward model score}. UnifiedReward~\cite{unifiedreward} is employed as the reward model and is found to generally align with human preferences. By effectively increasing low-frequency components (equivalent to suppressing high-frequency components under fixed latent energy), LCR improves generation quality.

\section{Experiments}
\subsection{Experiment Setup}
\textbf{Architecture.} We adopt a ResNet-based 3D VAE with 3D causal convolutions (similar to CogVideoX~\cite{cogvideox}) and extend DC-AE's~\cite{DC-AE} residual autoencoding to 3D for up/down sampling. We find that the architecture primarily impacts reconstruction quality, with limited effects on statistics like frequency or eigenspectrum. Default settings include $16\times$ spatial and $4\times$ temporal downsampling, a diffusion patch size of $(1\times 1\times 1)$, and a latent dimension of 48. This follows recent high-compression VAEs~\cite{wan}, which fold the usual $2\times2$ patchification of 16-channel, $8\times$ spatial-downsampled latents into the encoder. For existing video VAEs~\cite{wan,hunyuan-video,step-video,cogvideox}, the spatial patch size is set to 2 for those with $8\times$ spatial downsampling, and 1 otherwise.

\textbf{Data.} We mainly use open-source datasets to train VAEs: laion2B-en~\cite{laion} and COYO~\cite{coyo-700m} for images, and webvid~\cite{webvid} and panda70M~\cite{chen2024panda70m} for videos. For video generation, we train with an internal collection of movie and television clips, reserving 1,348 videos as a disjoint validation set, referred to as \textbf{MovieValid}. It covers diverse content, including landscapes, humans, animals, plants, architecture, artworks, black-and-white films, and cartoons.

\textbf{Evaluation.} Videos are generated with prompts from VBench (GPT-enhanced version)~\cite{vbench}, MovieGenBench~\cite{moviegenbench}, and MovieValid. We assess video quality using reward models and FVD. \revised{Given the well-documented shortcomings of FVD for T2V evaluation, including its preference for temporal inconsistency and artifacts~\cite{ge2024content}, as well as its inability to reflect multi-faceted quality~\cite{vbench}, we prioritize reward models such as UnifiedReward~\cite{unifiedreward} and VideoAlign~\cite{videoalign}, which achieve high human agreement (70–80\%)~\cite{unifiedreward}.} UnifiedReward~\cite{unifiedreward} scores are reported as `UR', and VideoAlign~\cite{videoalign} scores as `VAR'. Since VideoAlign’s raw rewards are unbounded logits, we apply the sigmoid function to each evaluation aspect and average the results to obtain a score in the 0–100 range.

\textbf{Training Details.} We train our VAE on a mixed image-video dataset with a $6{:}7$ ratio. The loss function combines $\Ls_{LMR}$, KL, LPIPS, GAN, and $\Ls_{LCR}$ with weights $1$, $5\mathrm{e}{-4}$, $1$, $1$, and $0.02$, respectively. For $\Ls_{LCR}$, we use the dynamic gradient-based weighting scheme from~\cite{VA-VAE} with threshold $\alpha=0.75$. For $\Ls_{LMR}$, mask ratios $\{0, 0.25, 0.5, 0.75\}$ are assigned selection probabilities $\{0.7, 0.1, 0.1, 0.1\}$. Prior to input to the diffusion backbone, VAE latents are normalized using the mean and standard deviation from MovieValid. Following SD3~\cite{SD3}, we train diffusion models under the flow-matching and velocity prediction setting, sampling timesteps from a logit-normal distribution.

\begin{table*}[t]
    \centering
    \caption{Generation comparison across various video VAEs. `CPR' refers to the spatial and temporal downsampling factors of the VAEs, and `Chn' is short for channel number. The best generation results are \textbf{bolded}, and the second-best are \underline{underlined}.}
\resizebox{\textwidth}{!}{%
    \begin{tabular}{c|c|ccc|ccc|ccc}
    \toprule
    \multirow{2}{*}{\textbf{Method}} & \multirow{2}{*}{\makecell{\textbf{CPR, Chn}}} & \multicolumn{3}{c|}{\bfseries VBench} & \multicolumn{3}{c|}{\bfseries MovieGenBench} & \multicolumn{3}{c}{\bfseries MovieValid}\\
    \cmidrule(lr){3-5} \cmidrule(lr){6-8} \cmidrule(lr){9-11}
    & & \bfseries UR$\uparrow$ & \bfseries VAR$\uparrow$ & \bfseries FVD$\downarrow$ & \bfseries UR$\uparrow$ & \bfseries VAR$\uparrow$ & \bfseries FVD$\downarrow$ & \bfseries UR$\uparrow$ & \bfseries \bfseries VAR$\uparrow$ & \bfseries FVD$\downarrow$\\
    \midrule
    \rowcolor{gray!10}
    & & \multicolumn{9}{c}{\textit{MMDiT 1.3B, $17\times 256\times 256$}}\\
    \midrule
    WF-VAE~\cite{WF-VAE}      & ($8\times 8\times 4$), 16 & 26.5 & 23.7 & \underline{978}  & 25.5 & 20.8 & 859 & 30.6 & 27.9 & 514\\
    IV-VAE~\cite{IV-VAE}      & ($8\times 8\times 4$), 16 & \underline{34.0} & \underline{27.8} & 1066 & \underline{28.8} & 22.0 & 784 & \underline{37.7} & \underline{30.9} & 541\\
    Hunyuan VAE~\cite{hunyuan-video} & ($8\times 8\times 4$), 16 & 29.7 & 25.9 & 1354 & 26.3 & 21.1 & 1169& 32.1 & 28.9 & 610\\
    CogvideoX VAE~\cite{cogvideox}   & ($8\times 8\times 4$), 16 & 30.4 & 26.6 & 1040 & 26.7 & 21.4 & 826 & 32.6 & 29.9 & 543\\
    Wan 2.1 VAE~\cite{wan} & ($8\times 8\times 4$), 16 & 33.3 & 27.2 & 986  & 28.5 & 22.1 & 765 & 36.5 & 30.8 & 517\\
    StepVideo VAE~\cite{step-video}  & ($16\times 16\times 8$), 64 & 29.7 & 25.6 & 1004 & 26.3 & 20.6 & 742 & 31.8 & 28.8 & \textbf{418}\\
    Wan 2.2 VAE~\cite{wan} & ($16\times 16\times 4$), 48 & 33.6 & \textbf{27.9} & \textbf{873}  & 28.4 & \textbf{22.8} & \underline{716} & 36.7 & 30.8 & \underline{432}\\
    Baseline                  & ($16\times 16\times 4$), 48 & 32.4 & 27.1 & 909 & 27.8 & 21.6 & 798 & 36.3 & 30.9 & 455\\
    \rowcolor{blue!8}
    SSVAE (Ours)              & ($16\times 16\times 4$), 48 & \textbf{34.7} & 27.7 & 983  & \textbf{30.2} & \underline{22.6} & \textbf{703} & \textbf{39.1} & \textbf{31.6} & 511\\
    \midrule
    \rowcolor{gray!10}
    & & \multicolumn{9}{c}{\textit{MMDiT 1.3B, $17\times 512\times 512$}}\\
    \midrule
    WF-VAE       & ($8\times 8\times 4$), 16 & 36.3 & 27.9 & 1352& 29.2 & 22.2 & 894 & 39.8 & 31.3 & 667\\
    IV-VAE       & ($8\times 8\times 4$), 16 & 41.8 & 29.8 & 997 & \underline{34.1} & 23.4 & 776 & 46.4 & 33.2 & 535\\
    Hunyuan VAE    & ($8\times 8\times 4$), 16 & 36.9 & 28.1 & 1166& 31.9 & 23.5 & 892 & 41.4 & 31.5 & 663\\
    CogvideoX VAE  & ($8\times 8\times 4$), 16 & 39.0 & 29.1 & \underline{971} & 31.2 & 23.4 & \underline{675} & 41.3 & 32.5 & \underline{452}\\
    Wan 2.1 VAE    & ($8\times 8\times 4$), 16 & 41.6 & 29.8 & 978 & 33.6 & \underline{23.7} & 791 & 46.6 & 33.9 & 500\\
    StepVideo VAE  & ($16\times 16\times 8$), 64 & 35.6 & 28.0 & 1190& 29.4 & 21.6 & 866 & 37.0 & 30.9 & 562\\
    Wan 2.2 VAE    & ($16\times 16\times 4$), 48 & \underline{42.8} & \underline{30.6} & 1019 & 33.4 & \underline{23.7} & 733 & \underline{47.9} & \underline{34.0} & 502\\
    Baseline     & ($16\times 16\times 4$), 48 & 40.9 & 29.1 & 970 & 32.9 & 23.3 & 681 & 46.5 & 33.9 & 504\\
    \rowcolor{blue!8}
    SSVAE (Ours) & ($16\times 16\times 4$), 48 & \textbf{45.9} & \textbf{30.7} & \textbf{828} & \textbf{35.7} & \textbf{24.3} & \textbf{605} & \textbf{50.8} & \textbf{34.8} & \textbf{403}\\
    \midrule
    \rowcolor{gray!10}
    & & \multicolumn{9}{c}{\textit{MMDiT 1.3B, $81\times 256\times 256$}}\\
    \midrule
    Wan 2.2 VAE & ($16\times 16\times 4$), 48 & 35.9 & 30.8 & 584 & 29.8 & 23.1 & 600 & 40.1 & 35.2 & \textbf{247}\\
    \rowcolor{blue!8}
    SSVAE (Ours) & ($16\times 16\times 4$), 48 & \textbf{37.8} & \textbf{33.2} & \textbf{560} & \textbf{31.6} & \textbf{24.5} & \textbf{484} & \textbf{40.8} & \textbf{36.8} & 272\\
    \midrule
    \rowcolor{gray!10}
    & & \multicolumn{9}{c}{\textit{Wan 1.3B, $17\times 512\times 512$}} \\
    \midrule
    Wan 2.2 VAE & ($16\times 16\times 4$), 48 & 37.7 & 27.3 & 1124 & 30.6 & 23.2 & 736 & 42.3 & 31.6 & 601\\
    \rowcolor{blue!8}
    SSVAE (Ours) & ($16\times 16\times 4$), 48 & \textbf{45.5} & \textbf{29.5} & \textbf{823} & \textbf{34.5} & \textbf{24.9} & \textbf{691} & \textbf{48.6} & \textbf{33.5} & \textbf{466}\\
    \bottomrule
    \end{tabular}
}
\label{tab:main_results}
\end{table*}
\subsection{Comparison with Existing Video VAEs}
We train 48-channel VAEs with LMR and LCR (`SSVAE (Ours)') and without them (`Baseline'). SSVAE is initially trained for 150k steps at $256 \times 256$ resolution. \revised{Subsequently, it is fine-tuned for 50k steps at $512 \times 512$ resolution with a fixed encoder and both LCR and LMR disabled to enhance high-resolution reconstruction.} Generation for $17\times 256\times 256$ and $81\times 256\times 256$ settings uses a 1.3B MMDiT-like~\cite{SD3} diffusion model trained for 50k steps on images, followed by a further 50k steps on 17-frame or 81-frame videos, respectively. For $17\times 512\times 512$, the model is initialized from the $17\times 256\times 256$ checkpoint and further trained for 100k steps on $512\times 512$ videos. \revised{A total of 200k generation steps matches or exceeds the training budget used in prior works~\cite{WF-VAE,IV-VAE}.}

\textbf{Generation Comparison across Resolutions and Backbones}. Spectra of 48-channel VAEs are shown in Fig.~\ref{fig:teaser}, while Tab.~\ref{tab:main_results} compares the generation quality of SSVAE to existing SOTA video VAEs. Key observations include:
(i) \textbf{SSVAE improves diffusability by a large margin}. Visual reward models align more closely with human judgments~\cite{unifiedreward}, and SSVAE achieves the highest reward scores across almost all settings and datasets. For $17\times 512\times 512$ generation, it achieves a $3\times$ UR convergence speedup over the 48-channel baseline (Fig.~\ref{fig:teaser}). Although FVD is mixed at 256p, SSVAE improves FVD at 512p and on 81-frame VBench/MovieGenBench, and the overall reward-model and human-preference results support SSVAE. 
(ii) \textbf{Generation quality (diffusability) is highly correlated with spectral properties}. Comparing the 48-channel Baseline, Wan 2.2, and SSVAE in Fig.~\ref{fig:teaser}, we observe that video reward scores consistently improve as spectral bias is enhanced, highlighting the significance of spectrum biasing;
(iii) \textbf{The advantages of SSVAE become more pronounced with longer denoising sequences}. Compared to the previous SOTA Wan 2.2 VAE, SSVAE yields an average improvement of 3.4 in UR and 1.4 in VAR across $17\times 512\times 512$ and $81\times 256\times 256$ generation tasks, and across all datasets. This may be primarily attributed to the ability of few-mode bias to accelerate convergence.
(iv) \textbf{SSVAE is diffusion-backbone-agnostic}. It consistently improves generative quality regardless of the diffusion backbone used, including both MMDiT~\cite{SD3} and Wan~\cite{wan}.

\textbf{Generation Performance across Model Scales}. We further compare SSVAE with Wan 2.2 VAE across various diffusion backbone model sizes, including 0.75B, 1.3B, and 4B parameters. As shown in Fig.~\ref{fig:teaser}, our experiments on VBench at $17\times 512\times 512$ demonstrate that SSVAE with a 1.3B diffusion backbone achieves a higher UnifiedReward score than Wan 2.2 VAE with a 4B diffusion model, while using 67.5\% fewer parameters. Fig.~\ref{fig:recons_gen_vis}(a) shows the generation results for the 4B models, where our method achieves higher visual quality and better alignment with the prompts.

\begin{figure}[t]
    \centering
    \includegraphics[width=0.85\linewidth]{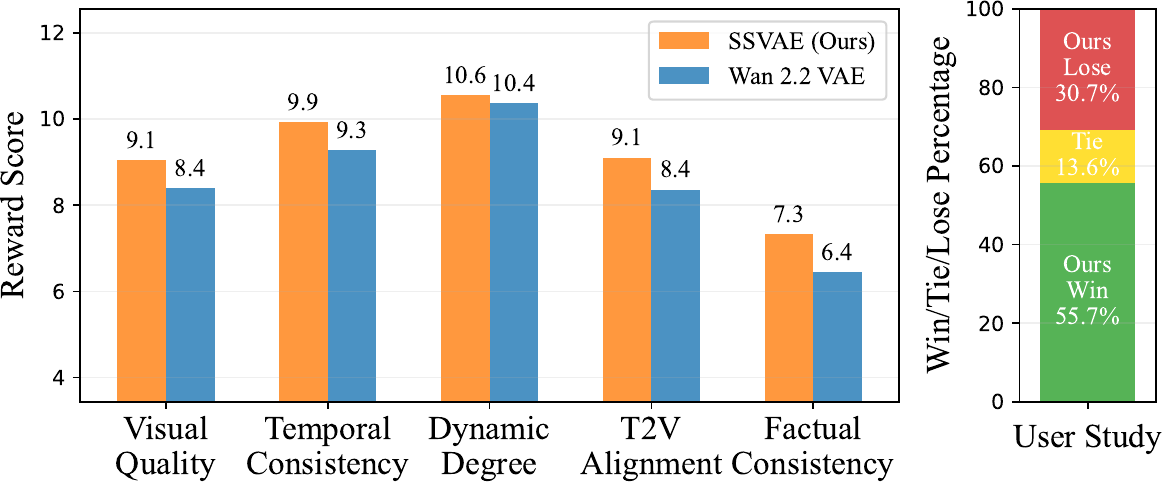}
    \caption{Human preference and UnifiedReward dimension breakdown.}
    \label{fig:breakdown_ur_user_study}
\end{figure}

\textbf{Human-Preference Validation}. Fig.~\ref{fig:breakdown_ur_user_study} details the ``MMDiT 1.3B, $17 \times 512 \times 512$'' setting from Tab.~\ref{tab:main_results}: SSVAE improves all five UnifiedReward dimensions and obtains a 55.7\% win rate in a 1,200-pair user study using MovieGenBench prompts.

\begin{table}[t]
    \centering
    \begin{minipage}[c]{0.61\linewidth}
        \centering
        \caption{Incremental component analysis on \revised{VBench} and MovieValid.}
        \resizebox{\linewidth}{!}{
            \begin{tabular}{lcccccc}
                \toprule
                \multirow{2}{*}{\bfseries Components} & 
                \multicolumn{3}{c}{\bfseries VBench} & 
                \multicolumn{3}{c}{\bfseries MovieValid} \\
                \cmidrule(lr){2-4} \cmidrule(lr){5-7}
                 & \bfseries UR & \bfseries VAR & \bfseries FVD 
                 & \bfseries UR & \bfseries VAR & \bfseries FVD \\
                \midrule
                Baseline & 36.3 & 26.8 & 1069 & 38.1 & 30.6 & 487\\
                +\ding{172}LCR & 36.8 & 27.3 & 993 & 39.0 & 30.6 & 472\\
                +\ding{172}+\ding{173}Covariance Penalize & 37.4 & 27.6 & 980 & 39.9 & 30.9 & 467\\
                +\ding{172}+\ding{173}+\ding{174}Decoder Finetune & 38.0 & 27.9 & 961 & 40.7 & 31.3 & 451\\
                +\ding{172}+\ding{175}LMR & \textbf{39.1} & \textbf{28.2} & \textbf{958} & \textbf{41.4} & \textbf{31.7} & \textbf{450}\\
                +\ding{172}+\ding{175}+\ding{174} & 38.9 & 28.0 & 960 & 41.3 & 31.6 & 450\\
                \bottomrule
            \end{tabular}
        }
        \label{tab:component_analysis}
    \end{minipage}
    \hfill
    \begin{minipage}[c]{0.37\linewidth}
        \centering
        \caption{Comparison of frequency spectrum biasing techniques. `LC' denotes local correlation computed on MovieValid.}
        \resizebox{\linewidth}{!}{
            \begin{tabular}{lcccc}
                \toprule
                \bfseries Components & \bfseries LC & \bfseries UR & \bfseries VAR & \bfseries FVD \\
                \midrule    
                64-ch Baseline & 0.49 & 34.0 & 28.4 & 553 \\
                +SER & 0.59 & 34.4 & 28.5 & \textbf{541} \\
                +VA-VAE & 0.73 & 34.6 & 28.6 & 546 \\
                \rowcolor{blue!8}
                +LCR & \textbf{0.80} & \textbf{35.5} & \textbf{29.2} & 606 \\
                \midrule
                48-ch Baseline & 0.67 & 38.1 & 30.6 & 487 \\
                \rowcolor{blue!8}
                +LCR & \textbf{0.80} & \textbf{39.0} & \textbf{30.6} & \textbf{472} \\
                \bottomrule
            \end{tabular}
        }
        \label{tab:freq_bias_ablation}
    \end{minipage}
\end{table}

\subsection{Ablation Studies}
All VAEs in the ablation experiments are trained for 100k steps at $256\times 256$ resolution, which we find sufficient for high-resolution generation, although it slightly compromises reconstruction quality. We then pre-train 1.3B MMDiT models on $512\times 512$ images for 40k steps, followed by finetuning on $17\times 512\times 512$ videos for 30k steps. Performance is evaluated mostly on the MovieValid benchmark, as its data distribution matches the training set and its prompts are derived from real videos.

\begin{figure*}[t]
  \centering
   \includegraphics[width=\linewidth]{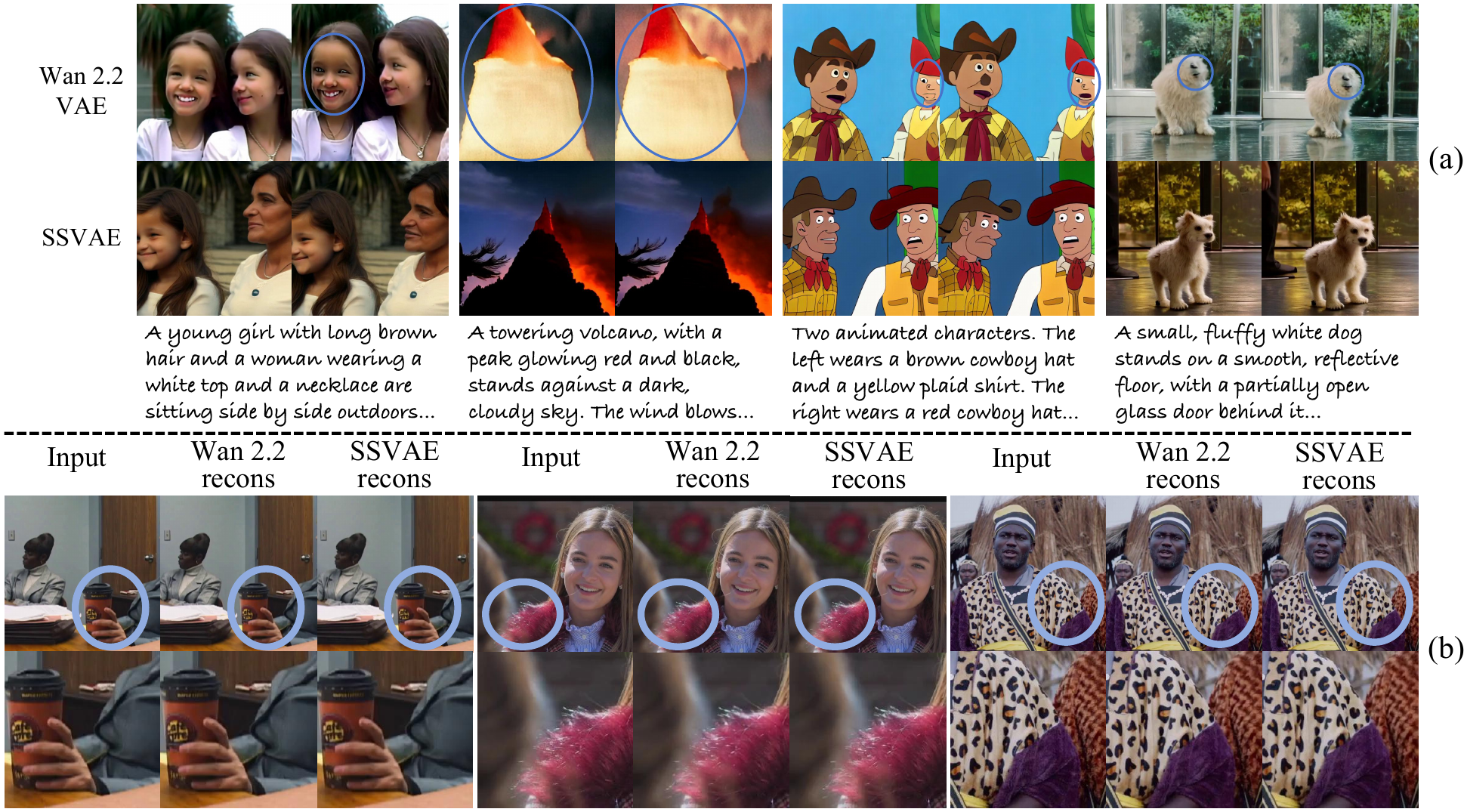}
   \caption{(a) A comparison of $17\times 512\times 512$ video generation using 4B diffusion models trained with SSVAE and Wan 2.2 VAE. For simplicity, only the first and last frames are shown. Prompts are sampled from MovieValid, and artifacts are highlighted in \textcolor{blue}{blue}. Our results demonstrate higher visual quality, fewer artifacts, and better alignment with the prompts. (b) Reconstruction comparison of SSVAE and Wan 2.2 VAE.}
   \label{fig:recons_gen_vis}
\end{figure*}

\textbf{Incremental Component Analysis.} We ablate the effects of low-frequency biasing, few-mode biasing, and decoder robustness using a 48-channel setting; results are shown in Tab.~\ref{tab:component_analysis}. `\textit{Covariance Penalize}' refers to the few-mode-forcing trick in Sec.~\ref{sec:FMB_impact}, where the sum of all eigenvalues except the largest three are minimized. `\textit{Decoder Finetune}' adds a training stage where the encoder is fixed and a Gaussian noise with $5\mathrm{e}-3$ variance is injected into the latents to finetune the decoder for improved robustness. We observe that: (i) Low-frequency bias, few-mode bias, and decoder finetuning complement each other in boosting generation quality; (ii) Simply applying decoder finetuning to a covariance-penalized VAE yields performance improvements. However, LMR achieves further gains by jointly promoting few-mode bias and enhancing decoder robustness through end-to-end training. Additional finetuning of the decoder after applying LMR offers no further benefit.
\label{sec:incre_comp_analysis}

\begin{table}[t]
\centering
\begin{minipage}[c]{0.42\textwidth}
\centering
\caption{Comparison of reconstruction results on MovieValid at a resolution of $17\times 512\times 512$.}
\resizebox{\textwidth}{!}{%
\begin{tabular}{l|ccc}
\toprule
\textbf{Method} & \textbf{PSNR}$\uparrow$ & \textbf{SSIM}$\uparrow$ & \textbf{LPIPS}$\downarrow$ \\
\midrule
Wan 2.1 VAE     & 38.57 & 0.97 & 0.0231\\
Wan 2.2 VAE     & 39.30 & 0.97 & 0.0206\\
WF-VAE          & 38.61 & 0.96 & 0.0352\\
IV-VAE          & 40.29 & 0.97 & 0.0220\\
Hunyuan VAE     & 40.31 & 0.97 & 0.0209\\
CogvideoX VAE   & 36.65 & 0.96 & 0.0398\\
StepVideo VAE   & 36.78 & 0.95 & 0.0435\\
\midrule
48-ch Baseline  & 38.07 & 0.96 & 0.0320\\
\rowcolor{blue!8}
SSVAE (Ours)    & 37.84 & 0.96 & 0.0325\\
\bottomrule
\end{tabular}
}
\label{tab:recons_comparison}
\end{minipage}\hspace{0.01\textwidth}
\begin{minipage}[c]{0.56\textwidth}
\centering
\caption{\revised{Impact of LCR and LMR on MovieValid reconstruction quality at a resolution of $17\times 512\times 512$. All VAEs are trained under the same regime.}}
\resizebox{\textwidth}{!}{%
\begin{tabular}{l|cccc}
\toprule
\textbf{Method} & \textbf{Param.} & \textbf{PSNR}$\uparrow$ & \textbf{SSIM}$\uparrow$ & \textbf{LPIPS}$\downarrow$ \\
\midrule
Wan 2.2 VAE-retrain  & 705M & 38.41 & 0.96 & 0.0312\\
\midrule
48-ch Baseline        & 315M & 38.07 & 0.96 & 0.0320\\
+LCR                  & 315M & 38.29 & 0.96 & 0.0315\\
+LCR+LMR              & 315M & 37.84 & 0.96 & 0.0325\\
\bottomrule
\end{tabular}
}
\label{tab:recons_ablation}
\end{minipage}
\vspace{0.5ex}

\begin{minipage}[c]{0.7\textwidth}
\centering
\caption{DAVIS reconstruction comparison at $512$p.}
\resizebox{0.6\textwidth}{!}{%
\begin{tabular}{lcc}
\toprule
\textbf{Method} & \textbf{PSNR}$\uparrow$ & \textbf{LPIPS}$\downarrow$ \\
\midrule
Wan 2.2 VAE & 30.90 & 0.056\\
Wan 2.2 VAE-retrain & 30.14 & 0.087\\
48-ch Baseline & 29.84 & 0.091 \\
SSVAE (Ours) & 29.69 & 0.091\\
\bottomrule
\end{tabular}
}
\label{tab:davis_recons}
\end{minipage}
\end{table}

\textbf{Frequency Biasing Ablation.} We revisit existing frequency biasing techniques on 64-channel VAEs, including scale-equivariant methods~\cite{SELoss} and foundation model alignment approaches~\cite{VA-VAE}. We also validate the effectiveness of LCR on 48-channel VAEs. As shown in Tab.~\ref{tab:freq_bias_ablation} and Fig.~\ref{fig:LCR}, we observe: (i) Local correlation positively correlates with generation quality: higher local correlation leads to better generation rewards; (ii) Low-frequency bias may explain the generative benefits of foundation model alignment. However, existing frequency biasing methods cannot sufficiently bias the spatio-temporal frequency spectrum, as they lack explicit mechanisms for handling the temporal dimension. In contrast, LCR efficiently introduces spatio-temporal low-frequency components and allows flexible control of their energy via the threshold $\alpha$; (iii) The improvements from LCR are less pronounced in low-channel VAEs compared to high-channel VAEs. We attribute this to low-channel VAEs naturally introducing more low-frequency components. As shown in Tab.~\ref{tab:freq_bias_ablation}, the local correlation of the 48-channel baseline (0.67) is higher than that of the 64-channel baseline (0.49).
\label{sec:psd_exp}

\textbf{Reconstruction Comparison and Ablation}. We compare SSVAE and other VAEs at $17\times 512\times 512$ in Tab.~\ref{tab:recons_comparison}. While SSVAE's reconstruction fidelity is not the highest, it surpasses StepVideo and CogVideoX, demonstrating that its reconstruction capability is enough for training high-quality video diffusion models. Visualization results in Fig.~\ref{fig:recons_gen_vis}(b) confirm that its reconstruction remains \textbf{visually faithful, avoiding perceptible semantic distortions in details} such as the cup's text (Cafe) and clothing patterns. We ablate the impact of LCR and LMR in Tab.~\ref{tab:recons_ablation} under identical training regimes and observe: (i) \textbf{Smooth latents do not necessitate RGB detail loss.} The LCR, which induces latent smoothness, even slightly improves reconstruction. This aligns with existing 16-channel models like IV-VAE and Wan 2.1, which exhibit similarly steep PSDs (Supp.~Fig.~\ref{fig:psd_comp_for_lcr}) yet maintain high fidelity. (ii) LMR incurs a mild quality drop (0.23 PSNR vs. 48-ch baseline), which likely stems from a trade-off between decoding precision and robustness. This gap further narrows on DAVIS, where SSVAE is only 0.15 dB PSNR below the 48-ch baseline with matched LPIPS (Tab.~\ref{tab:davis_recons}). We further train downstream I2V models with channel-concat conditioning and find that SSVAE remains close to Wan 2.2 VAE in first-frame PSNR (34.95 vs. 35.14, gap 0.19), suggesting that diffusion noise reduces the practical PSNR gap between VAEs. (iii) The disparity with the official Wan 2.2 VAE likely arises from its greater parameter count and more extensive training (scale and data). For context, the comparable Seaweed~\cite{seaweed} underwent 1.3M training steps, whereas our model used 200k. Notably, the gap narrows significantly when comparing SSVAE to the re-trained Wan 2.2 baseline. Although extended training also improves SSVAE's reconstruction, we observed \textbf{negligible gains in generation quality and thus prioritized computational efficiency}.

\textit{In the supplementary material, experiments on public datasets confirm that our gains are independent of diffusion training data. Additionally, we provide LCR and LMR hyperparameter and computational overhead analyses, CPR/LMR residual diagnostics, relations of spectral regularization to spatial representation regularizations, extra visualizations, and the UnifiedReward-Thinking evaluation prompt.}

\section{Conclusion}
We analyze video VAE latent spaces and identify two critical spectral properties for diffusion training: few-mode bias in the channel eigenspectrum and low-frequency bias in the spatio-temporal spectrum. We introduce two lightweight regularizers to induce these properties, and experiments show that our approach accelerates convergence and improves video generation quality.

\bibliographystyle{splncs04}
\bibliography{main}

\clearpage
\appendix
\section*{Supplementary Material}
\setcounter{section}{0}
\setcounter{figure}{0}
\setcounter{table}{0}
\setcounter{equation}{0}
\renewcommand{\thesection}{A.\arabic{section}}
\renewcommand{\thefigure}{A.\arabic{figure}}
\renewcommand{\thetable}{A.\arabic{table}}
\renewcommand{\theequation}{A.\arabic{equation}}
\renewcommand{\theHsection}{appendix.\arabic{section}}
\renewcommand{\theHfigure}{appendix.\arabic{figure}}
\renewcommand{\theHtable}{appendix.\arabic{table}}
\renewcommand{\theHequation}{appendix.\arabic{equation}}

\begin{itemize}
    \item Sec.~\ref{sec:proof}: Proof of Theorem 1.
    \item Sec.~\ref{sec:autocorrelation}: Relationship among Small-lag Autocorrelation, Low-frequency Energy, and LCR.
    \item Sec.~\ref{sec:recon_vis}: Compatibility of Smooth Latents and High Reconstruction Quality.
    \item Sec.~\ref{sec:public_data_validation}: Public-Data Validation.
    \item Sec.~\ref{sec:dino_affinity}: DINO Structural Affinity Diagnostic.
    \item Sec.~\ref{sec:lcr_hyperparameters}: Impact of LCR Hyperparameters on PSD.
    \item Sec.~\ref{sec:lmr_hyperparameters}: Impact of LMR Mask Hyperparameters.
    \item Sec.~\ref{sec:cpr_lmr_residual}: CPR/LMR Residual CEV Diagnostic.
    \item Sec.~\ref{sec:Principal}: Visualization of Latent Principal Components.
    \item Sec.~\ref{sec:visualization}: More Generation Results.
    \item Sec.~\ref{sec:Computation}: Computational Overhead of LCR and LMR.
    \item Sec.~\ref{sec:Prompts}: Prompts for UnifiedReward-Thinking.
\end{itemize}

\section{Proof of Theorem 1}
\label{sec:proof}
We briefly recall the notation and Theorem 1. For a $C$-channel VAE, let $\rvu^0 \in \mathbb{R}^{1\times C}$ denote a latent vector sampled from the standardized latent $\tilde{\rvz}$, where the superscript $0$ denotes the clean latent at diffusion timestep $0$. The channel-wise autocorrelation matrix $\Sigma_{\rvu \rvu} \in \mathbb{R}^{C \times C}$ is defined as $\Sigma_{\rvu \rvu} = \mathbb{E}[(\rvu^0)^\top \rvu^0]$. Let $\Sigma_{\rvv \rvu}(t) \in \mathbb{R}^{C \times C}$ denote the channel-wise output–input cross-correlation matrix of the diffusion backbone at timestep $t$, and let $\lambda_l$ denote the $l$-th largest eigenvalue of $\Sigma_{\rvu\rvu}$. Then we have the following theorem under flow-matching with velocity prediction.
\begin{figure}[t]
    \centering
    \includegraphics[width=0.6\linewidth]{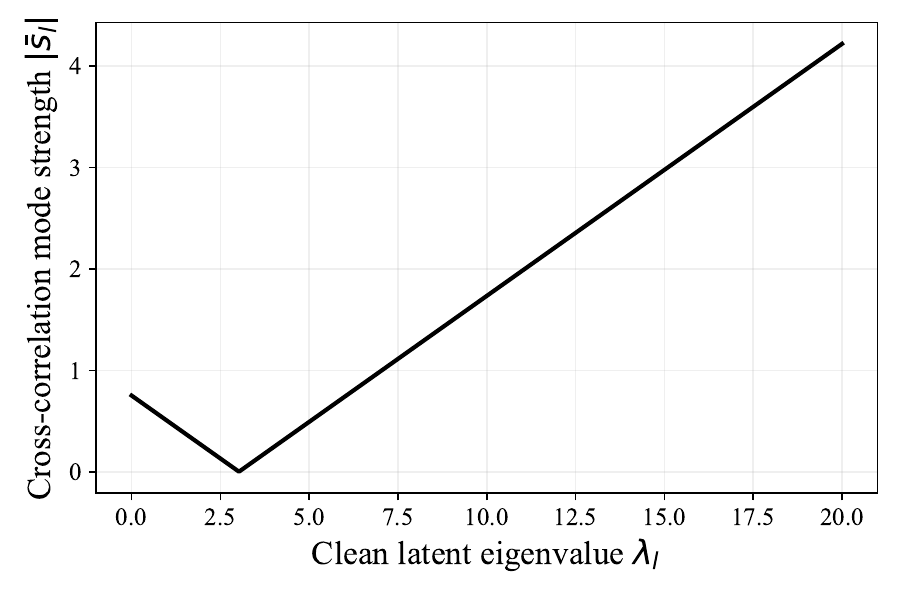}
    \caption{Relationship between the eigenvalue $\lambda_l$ of the clean latent autocorrelation matrix and its corresponding expected absolute mode strength $|\bar{s}_l|$ in the diffusion output-input cross-correlation matrix. Timesteps are sampled from logit-normal distribution with logit mean 0.0 and logit std 1.0.}
    \label{fig:lambda_to_si}
\end{figure}
\begin{theorem}
$\Sigma_{\rvv\rvu}(t)$ has the same eigenvectors as $\Sigma_{\rvu\rvu}$, and its eigenvalue on the $l$-th eigenvector of $\Sigma_{\rvu\rvu}$ is given by:
\begin{equation*}
s_l(t) = t - (1-t)\lambda_l.
\end{equation*}
\end{theorem}
\begin{proof}
Under the flow-matching setting with velocity prediction, we have
\begin{align*}
\rvu^t &= (1-t)\,\rvu^0 + t\,\boldsymbol{\epsilon},\; 
\rvv = \boldsymbol{\epsilon} - \rvu^0,\; 
\boldsymbol{\epsilon} \sim \mathcal{N}(\boldsymbol{0}, \mI),
\end{align*}
where $\rvu^t$ is the diffusion input at timestep $t$, $\rvv$ is the velocity to be predicted, and $\boldsymbol{\epsilon}$ is standard Gaussian noise. The cross-correlation matrix is
\begin{align}
\Sigma_{\rvv \rvu}(t) 
&= \mathbb{E}\big[\rvv^\top \rvu^t\big] 
= \mathbb{E}\big[(\boldsymbol{\epsilon}-\rvu^0)^\top\big((1-t)\rvu^0 + t\boldsymbol{\epsilon}\big)\big] \notag \\
&= (1-t)\,\mathbb{E}\!\big[\boldsymbol{\epsilon}^\top \rvu^0\big] - t\,\mathbb{E}\!\big[(\rvu^0)^\top \boldsymbol{\epsilon}\big] \notag \\
& \quad + t\,\mathbb{E}\!\big[\boldsymbol{\epsilon}^\top \boldsymbol{\epsilon}\big]
- (1-t)\,\mathbb{E}\!\big[(\rvu^0)^\top \rvu^0\big] \notag \\
&= t\,\mI - (1-t)\,\Sigma_{\rvu \rvu}, \label{eq:covariance_relation}
\end{align}
where we used the independence of $\rvu^0$ and $\boldsymbol{\epsilon}$ together with the zero mean of $\boldsymbol{\epsilon}$ to drop the cross terms, and $\mathbb{E}[\boldsymbol{\epsilon}^\top \boldsymbol{\epsilon}] = \mI$. Since $\Sigma_{\rvv \rvu}(t)$ is an affine function of $\mI$ and $\Sigma_{\rvu \rvu}$, it is diagonalized by the same orthogonal matrix as $\Sigma_{\rvu \rvu}$. Hence they share the same eigenvectors, and the corresponding eigenvalues satisfy
\begin{equation*}
s_l(t) = t - (1-t)\lambda_l.
\end{equation*}
\end{proof}

We visualize in Fig.~\ref{fig:lambda_to_si} the relationship between the eigenvalues $\lambda_l$ of the clean latent autocorrelation matrix and their corresponding expected absolute cross-modal strengths $\bar{s}_l$, with timesteps sampled from a commonly used logit-normal distribution (logit mean $0.0$, logit std $1.0$). As shown in the figure, the relationship between $\lambda_l$ and $\bar{s}_l$ is a piecewise linear function, where $\bar{s}_l$ first decreases and then increases as $\lambda_l$ grows. This indicates that very small eigenvalues in the clean latent can correspond to relatively large absolute mode strengths in the cross-correlation matrix.

The theorem is stated for velocity-prediction flow matching, which is the training objective used by many recent T2V systems. Its main implication is not tied to this exact parameterization: for any linear noising process with an affine prediction target, the target-input cross-correlation remains a linear combination of the identity matrix and the clean-latent covariance. Therefore, the clean-latent eigenspectrum still determines the modal structure seen by the diffusion backbone, up to changes in the coefficients attached to each mode. In contrast, under $\boldsymbol{\epsilon}$-prediction, the target-input cross-correlation reduces to an isotropic signal, so different latent covariance modes no longer receive different target-input strengths from this diagnostic.

\section{Relationship among Small-lag Autocorrelation, Low-frequency Energy, and LCR}
\label{sec:autocorrelation}
We begin by analyzing the relationship between small-lag autocorrelation and low-frequency energy, then illustrate how LCR approximates small-lag autocorrelation.

Let $\tilde{\rvz}\in\R^{T\times H\times W\times C}$ denote the per-channel standardized VAE latent with spatio-temporal indices $(i, j, k)$. The channel-summed autocorrelation function is defined as
\begin{equation}
R[\delta_i,\delta_j,\delta_k] \;=\; 
\,\E_{i,j,k}\big[\tilde{\rvz}^\top_{i,j,k}\,\tilde{\rvz}_{i+\delta_i,\,j+\delta_j,\,k+\delta_k}\big].
\label{eq:autocorrelation}
\end{equation}
where the inner expectation is over valid indices such that both $(i,j,k)$ and $(i+\delta_i,j+\delta_j,k+\delta_k)$ lie in the latent domain.

To make the link to spectral energy transparent, we consider the one-dimensional case. Let the dimension length be $N$ and write the autocorrelation as $R[\delta]$ for $\delta\in\{0,\ldots,N-1\}$. By the discrete Wiener–Khinchin theorem, the power spectral density (PSD) $S[m]$ is the discrete Fourier transform (DFT) of $R[\delta]$:
\begin{equation}
S[m] = \sum_{\delta=0}^{N-1} R[\delta] \mathrm{e}^{-\,\mathrm{i}\,2\pi\frac{m\delta}{N}},\; m\in \{0, 1, ..., N-1 \}.
\end{equation}

\begin{figure}[t]
\centering
\includegraphics[width=0.8\linewidth]{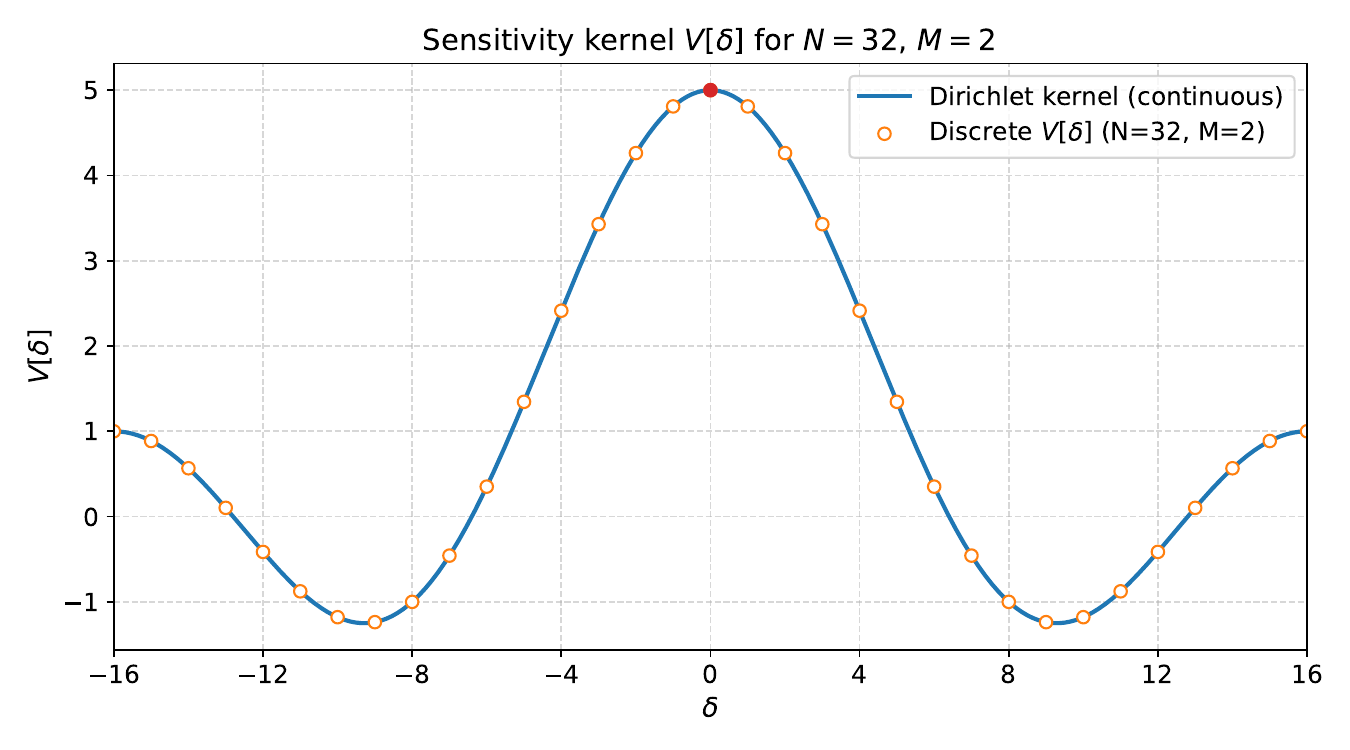}
\caption{The sensitivity kernel $V[\delta]$ for $N=32$, $\mathbb{M}^{\mathrm{low}}_{max}=2$ has a sharp peak at small lags and exhibits oscillatory decay with increasing $|\delta|$, indicating its low-pass property.}
\label{fig:sensitivity_kernel}
\end{figure}

We then define the low-frequency energy by introducing a low-frequency index set $\mathbb{M}^{\mathrm{low}}$. For real-valued signals, it is convenient to take a symmetric low-frequency band around DC, e.g., $\mathbb{M}^{\mathrm{low}}=\{0,1,\ldots, \mathbb{M}^{\mathrm{low}}_{max} \}\cup\{N-\mathbb{M}^{\mathrm{low}}_{max},\ldots,N-1\}$ for some $\mathbb{M}^{\mathrm{low}}_{max}\ll N$. Then
\begin{equation}
E_{\mathrm{low}} = \sum_{m\in \mathbb{M}^{\mathrm{low}}} S[m] = \sum_{m\in \mathbb{M}^{low}} \sum_{\delta = 0}^{N-1} R[\delta]\;\mathrm{e}^{-\,\mathrm{i}\,2\pi\frac{m\delta}{N}}.
\end{equation}
Exchanging the sums, we obtain
\begin{equation}
E_{\mathrm{low}} = \sum_{\delta=0}^{N-1} R[\delta]\cdot V[\delta],
\end{equation}
where the sensitivity kernel $V[\cdot]$ is
\begin{align}
V[\delta]
&= \sum_{m\in \mathbb{M}^{\mathrm{low}}} \mathrm{e}^{-\,\mathrm{i}\,2\pi\frac{m\delta}{N}}
= 1 + 2\sum_{m=1}^{\mathbb{M}^{\mathrm{low}}_{max}} \cos\!\Big(2\pi\frac{m\delta}{N}\Big) \notag\\
&= \frac{\sin\!\big((2\mathbb{M}^{\mathrm{low}}_{max}+1)\,\pi\,\delta/N\big)}{\sin\!\big(\pi\,\delta/N\big)}\, \notag \\
& \text{with }V[0]=2\mathbb{M}^{\mathrm{low}}_{max}+1.
\end{align}
Thus, the low-frequency energy $E_{\mathrm{low}}$ is a weighted sum (inner product) of the autocorrelation sequence $R[\delta]$ with a Dirichlet-type kernel $V[\delta]$. 
The kernel $V[\delta]$ is sharply peaked at $\delta=0$ and decays away from zero with oscillatory side lobes, as depicted in Fig.~\ref{fig:sensitivity_kernel}. Consequently, when $R[\delta]$ is absolutely summable and its magnitude decreases with $|\delta|$, which is a mild property satisfied by typical natural spatio-temporal data, then the contribution of small lags to $E_{\mathrm{low}}$ dominates. This formalizes the intuition that low-frequency energy is primarily controlled by the values of the autocorrelation at small lags.

We now connect this to our Local Correlation Regularization. Recall the local correlation within a patch $p$:
\begin{align}
\tilde R(p) \;=\; & \E_{\substack{(i,j,k),(\hat{i}, \hat{j}, \hat{k})\in p,\\ (i,j,k)\neq(\hat{i}, \hat{j}, \hat{k})}} \big[\, \tilde{\rvz}^\top_{i,j,k}\, \tilde{\rvz}_{\hat{i},\hat{j},\hat{k}} \,\big].
\label{eq:supp_local_corres}
\end{align}
The key difference between Eq.~\ref{eq:autocorrelation} and Eq.~\ref{eq:supp_local_corres} is that Eq.~\ref{eq:supp_local_corres} restricts the pairs to lie within the same spatio-temporal patch and discards inter-patch pairs. Since pairs within a small patch necessarily correspond to small spatio-temporal offsets, the terms aggregated by $\tilde R(p)$ form a subset of the small-lag terms of autocorrelation. Thus, averaging $\tilde R(p)$ over patches yields a biased but consistent proxy for the small-lag autocorrelation, while avoiding the computational cost of enumerating all pairs across the full volume.

Maximizing local correlations therefore encourages increases in small-lag autocorrelation, which raises low-frequency energy in the latents. In practice, we find that replacing the per-pair dot product with cosine similarity does not change the qualitative optimization objective with respect to autocorrelation, while it improves numerical stability and prevents high-variance patches from dominating. We therefore adopt this modification in LCR for a more balanced and robust optimization.

\section{Compatibility of Smooth Latents and High Reconstruction Quality}
\label{sec:recon_vis}
\begin{figure}[t]
    \centering
    \includegraphics[width=0.5\linewidth]{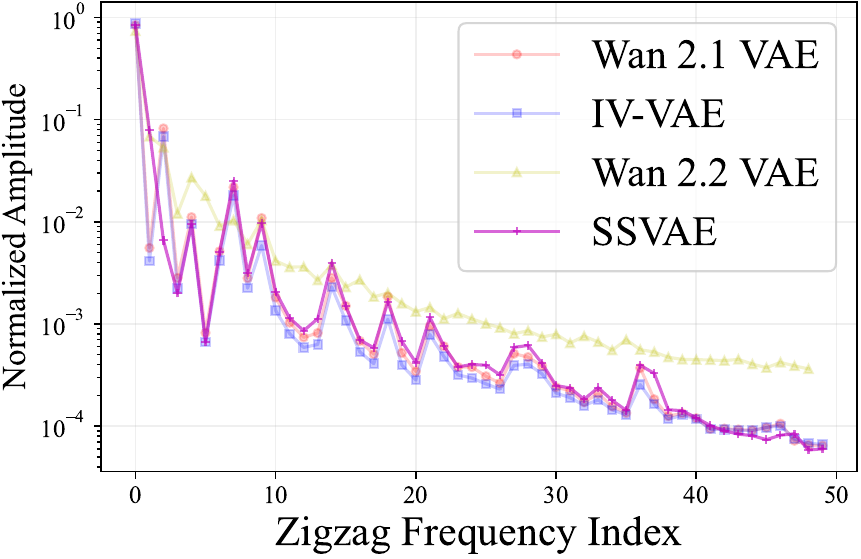}
    \caption{PSD comparison between existing Video VAEs and SSVAE.}
    \label{fig:psd_comp_for_lcr}
\end{figure}

\begin{figure}[t]
    \centering
    \includegraphics[width=\linewidth]{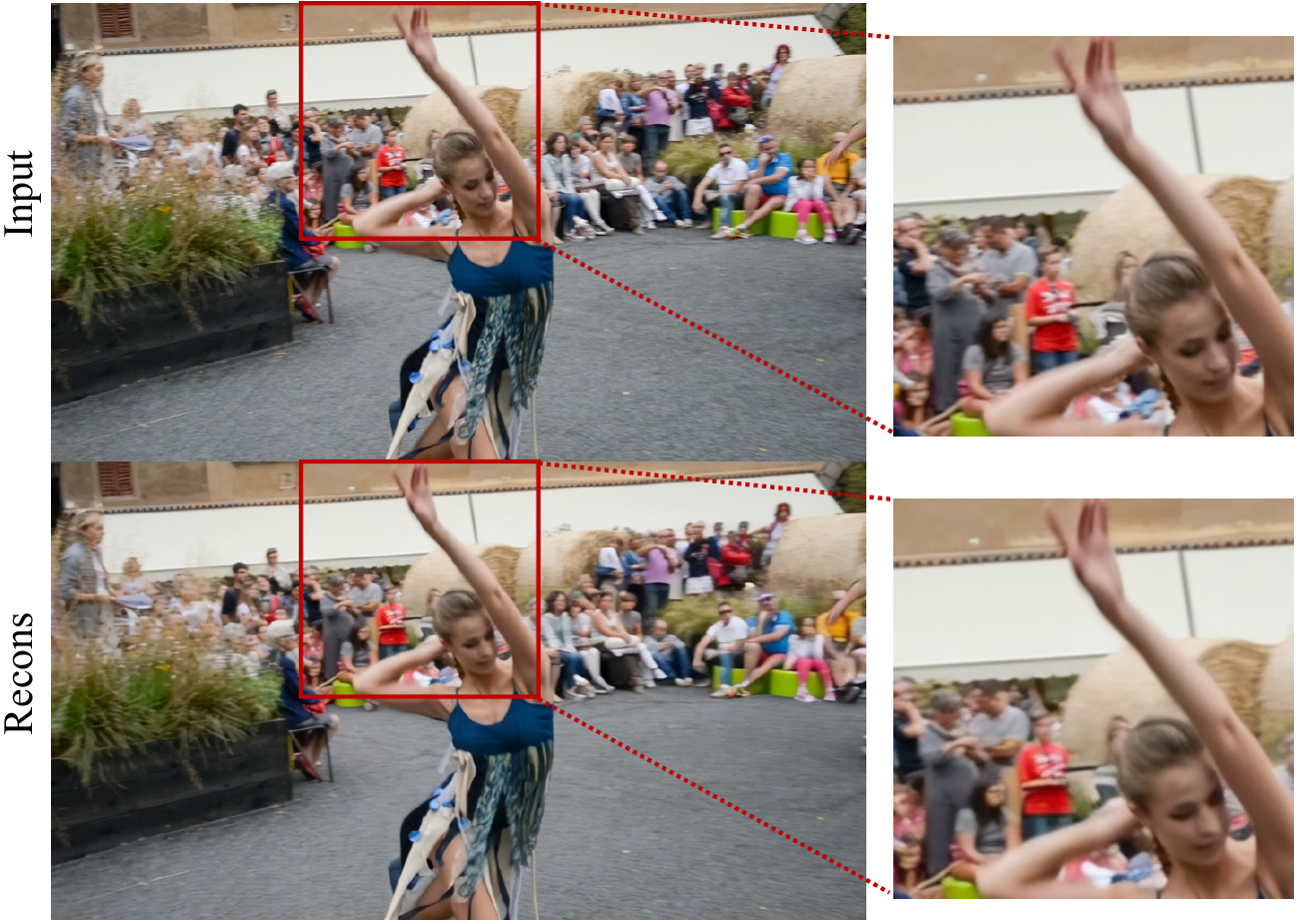}
    \caption{DAVIS reconstruction visualization.}
    \label{fig:davis_recons}
\end{figure}

We analyze the PSD curves of SSVAE and existing video VAEs in Fig.~\ref{fig:psd_comp_for_lcr}. Benefiting from the 16-channel design, IV-VAE and Wan 2.1 VAE exhibit PSD steepness similar to SSVAE, yet both maintain high reconstruction fidelity (IV-VAE: 40.29 dB, Wan 2.1: 38.57 dB). This indicates that smooth latents do not necessarily compromise RGB reconstruction quality. We provide the DAVIS quantitative results and summarize the I2V first-frame diagnostic in the main paper, and visualize DAVIS reconstructions in Fig.~\ref{fig:davis_recons}. The visual examples show no obvious perceptual degradation, supporting the conclusion that the mild reconstruction difference does not impose a clear bottleneck on downstream generation.

\section{Public-Data Validation}
\label{sec:public_data_validation}
To rule out internal training data bias, we train diffusion models with SSVAE and Wan 2.2 VAE on public datasets COYO, LAION, and Panda70M for 100k steps at $512 \times 512$ resolution. As shown in Tab.~\ref{tab:public_training}, SSVAE consistently outperforms Wan 2.2 VAE on VBench, matching the trend in the main paper.

\begin{table}[t]
    \centering
    \caption{Generation quality when diffusion models are trained on public datasets.}
    \begin{tabular}{lccc}
        \toprule
        Method & UR & VAR & FVD \\
        \midrule
        Wan 2.2 & 38.7 & 27.7 & 1209 \\
        SSVAE   & \textbf{40.7} & \textbf{29.1} & \textbf{1092} \\
        \bottomrule
    \end{tabular}
    \label{tab:public_training}
\end{table}

\section{DINO Structural Affinity Diagnostic}
\label{sec:dino_affinity}
To clarify how our spectral regularizers relate to spatial-structure-oriented representation regularization, we report a DINO structural affinity diagnostic in Tab.~\ref{tab:dino_affinity}. The metric is the Spearman correlation between the pairwise token-similarity matrix of VAE latents and that of aligned DINOv2 patches.

\begin{table}[t]
    \centering
    \caption{DINO structural affinity diagnostic.}
    \begin{tabular}{lc}
        \toprule
        Method & Latent-DINO corr.\\
        \midrule
        48-ch Baseline & 0.378 \\
        +LCR & 0.408\\
        +LCR + LMR & 0.410\\
        \bottomrule
    \end{tabular}
    \label{tab:dino_affinity}
\end{table}

The diagnostic suggests a partial overlap between LCR and spatial-structure-oriented regularization. A moderate LCR improves latent-DINO correlation, which is consistent with the intuition that encouraging neighboring latent tokens to be locally coherent also improves spatial structure. This is not a competing explanation to the spectral view: the same neighboring-token similarity controls small-lag autocorrelation and therefore low-frequency energy.

In contrast, LMR only marginally changes the latent-DINO correlation beyond LCR. We therefore do not attribute LMR's gain mainly to improved spatial-structure alignment. Its role is better understood as channel eigenspectrum shaping: LMR promotes a controlled few-mode bias and changes the modal structure seen by the diffusion backbone, complementing LCR along a different axis.

\begin{figure}[t]
    \centering
    \begin{minipage}[c]{0.53\linewidth}
        \centering
        \includegraphics[width=\linewidth]{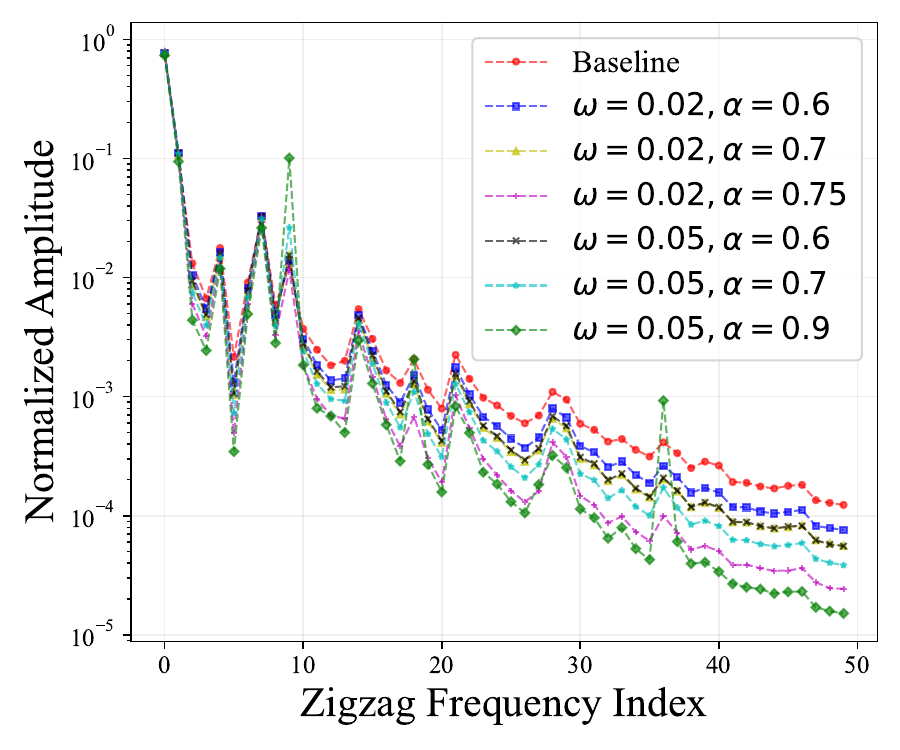}
        \captionof{figure}{Effect of varying the LCR loss weight $\omega$ and threshold $\alpha$ on the power spectral density curve.}
        \label{fig:lcr_thresholds}
    \end{minipage}
    \hfill
    \begin{minipage}[c]{0.46\linewidth}
        \centering
        \captionof{table}{Effect of varying the LCR loss weight $\omega$ and threshold $\alpha$ on the generation quality.}
        \label{tab:lcr_hyperparam}
            \begin{tabular}{ccccc}
                \toprule
                \bfseries $\omega$ & \bfseries $\alpha$ & \bfseries UR & \bfseries VAR & \bfseries FVD \\
                \midrule
                0.00 & - & 38.1 & 30.6 & 487 \\
                0.05 & 0.6 & 38.5 & 30.6 & 453 \\
                0.02 & 0.75 & \textbf{39.0} & 30.6 & 472 \\
                0.02 & 0.90 & 38.0 & \textbf{30.7} & \textbf{438} \\
                \bottomrule
            \end{tabular}
    \end{minipage}
\end{figure}

\section{Impact of LCR Hyperparameters on PSD}
\label{sec:lcr_hyperparameters}
We study the impact of LCR hyperparameters on the power spectral density (PSD) curve and generation quality. The VAEs are trained for 100k steps at a resolution of $256\times 256$, consistent with the setting used in the Ablation Studies. As shown in Fig.~\ref{fig:lcr_thresholds} and Tab.~\ref{tab:lcr_hyperparam}, we observe the following:  
(i) The steepness of the PSD curve can be flexibly controlled by varying the hyperparameters of LCR. Increasing either the loss weight $\omega$ or the local correlation threshold $\alpha$ leads to a gradual increase in PSD steepness. For precise control of the local correlation using LCR, we recommend adjusting the threshold $\alpha$ while keeping the loss weight $\omega$ sufficiently large and fixed, e.g., at 0.02. 
(ii) Setting $\omega$ and $\alpha$ too high can result in mid-frequency spikes in the PSD curve, see `$\omega=0.05, \alpha=0.9$' in Fig.~\ref{fig:lcr_thresholds} and Tab.~\ref{tab:lcr_hyperparam}. This occurs because small-lag autocorrelations also contribute to other frequency components with low weight, and excessively high small-lag autocorrelation can anomalously amplify the energy of these undesired components. In practice, we find that an over-biased PSD curve does not provide further improvements in diffusability and may even be detrimental. Accordingly, we ultimately set the hyperparameters to $\omega=0.02$ and $\alpha=0.75$, which achieve the highest UR.

\begin{table}[t]
    \centering
    \caption{Effect of mask ratios and selection probabilities in LMR on generation quality.}
    \begin{tabular}{ccccc}
        \toprule
        \bfseries Mask Ratios & \bfseries  Selection Probabilities & \bfseries UR & \bfseries VAR & \bfseries FVD \\
        \midrule
        - & - & 39.2 & 30.6 & \textbf{471} \\
        $\{0,0.25,0.5,0.75\}$ & $\{0.6,0.1,0.15,0.15\}$ & 41.3 & 31.2 & 475 \\
        $\{0,0.25,0.5,0.75\}$ & $\{0.7,0.1,0.1,0.1\}$ & \textbf{42.9} & \textbf{32.0} & 489 \\
        $\{0,0.2,0.4,0.6\}$ & $\{0.7,0.1,0.1,0.1\}$ & 39.8 & 31.0 & 512 \\
        $\{0,0.25,0.5,0.75\}$ & $\{0.8,0.1,0.05,0.05\}$ & 39.6 & 31.0 & 476 \\
        \bottomrule
    \end{tabular}
    \label{tab:lmr_hyperparam}
\end{table}

\section{Impact of LMR Mask Hyperparameters}
\label{sec:lmr_hyperparameters}
We analyze the sensitivity of LMR to two hyperparameters: the mask ratios and their selection probabilities, results are shown in Table~\ref{tab:lmr_hyperparam}. The first row corresponds to the baseline with LCR only, while all other rows report results with both LCR and LMR enabled. The first row denotes the baseline without LMR. Diffusion models are trained on $512\times 512$ images for 40k steps, followed by finetuning on $17\times 512\times 512$ videos for 30k steps. The VAE here is trained at a resolution of $256\times 256$ for 150k steps, consistent with our $256\times 256$ training setup in Tab.~\ref{tab:main_results}.

Across different ratio sets and probability allocations, LMR consistently outperforms the baseline without LMR, indicating its robustness to variations in the mask scheduling. As shown from the second to the last row in Tab.~\ref{tab:lmr_hyperparam}, as the expected mask ratio decreases during training, both UR and VAR first increase and then decrease. This pattern aligns with intuition: an excessively low expected mask ratio fails to encourage sufficient few-mode bias, while over-emphasizing large mask ratios indeed strengthens the few-mode bias but tends to compromise reconstruction quality, thereby degrading UR and VAR. Based on these observations, we adopt $\{0, 0.25, 0.5, 0.75\}$ as mask ratios with selection probabilities $\{0.7, 0.1, 0.1, 0.1\}$ as the default configuration, achieving a balance between reconstruction fidelity and few-mode bias.

\section{CPR/LMR Residual CEV Diagnostic}
\label{sec:cpr_lmr_residual}
To further compare LMR with Channel-wise Progressive Reconstruction (CPR), we apply each perturbation to 48-ch baseline latents and measure the channel covariance of the resulting corruption residuals, namely the latent differences induced by each perturbation. Tab.~\ref{tab:cpr_lmr_residual} reports the cumulative explained variance (CEV) of these residuals.
\begin{table}[t]
    \centering
    \caption{Channel covariance of CPR/LMR corruption residuals.}
    \begin{tabular}{lcc}
        \toprule
        Method & CEV@1 & CEV@3\\
        \midrule
        CPR & 0.42 & 0.61\\
        LMR & 0.40 & 0.69\\
        \bottomrule
    \end{tabular}
    \label{tab:cpr_lmr_residual}
\end{table}

LMR has a higher CEV@3 than CPR, indicating broader coverage over the leading residual subspace. This complements the PCA comparison in the main paper: CPR is more concentrated on the top mode, while LMR spreads corruption across multiple leading modes. Such a residual pattern better matches LMR's goal of promoting a controlled few-mode bias rather than only amplifying the single dominant direction.

\begin{figure*}[t]
    \centering
    \includegraphics[width=\linewidth]{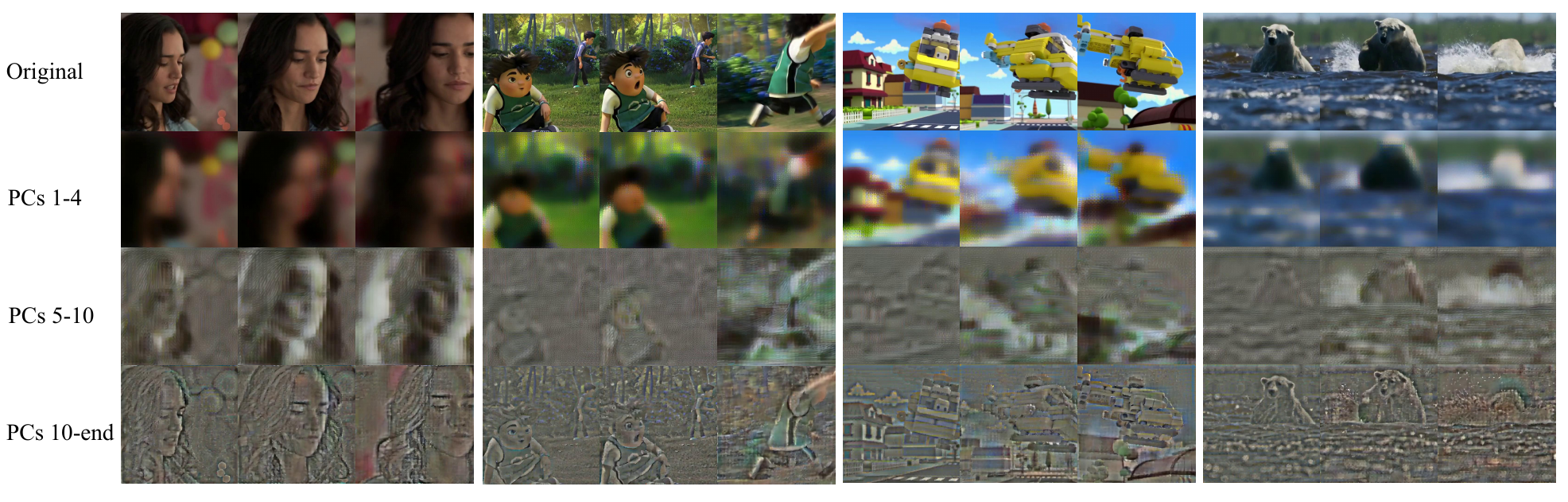}
    \caption{Visualization of information encoded in different principal components (PC).}
    \label{fig:PCs}
\end{figure*}
\section{Visualization of Latent Principal Components}
\label{sec:Principal}
To better understand the role of principal components (PC) in the channel autocorrelation matrix of the clean latent, we visualize in Fig.~\ref{fig:PCs} the information encoded by different principal components (eigenvectors). Specifically, we first perform eigendecomposition on the autocorrelation matrix to obtain the eigenvalues and their corresponding eigenvectors, sorted in descending order. For each VAE latent, we then project it onto the subspace spanned by a subset of eigenvectors (e.g., the first 1–4 vectors), and the projected latent is then fed into the VAE decoder for reconstruction. 

It can be observed that eigenvectors corresponding to larger eigenvalues (PCs 1–4) primarily represent low-frequency components in the video, such as blurred contours, coarse motion, and color. Additionally, eigenvectors associated with smaller eigenvalues (PCs 10–end) encode high-frequency edge features and motion information. Therefore, it is important for diffusion models to learn both high-eigenvalue and low-eigenvalue modes.

\section{More Generation Results}
\label{sec:visualization}
\begin{figure*}[t]
    \centering
    \includegraphics[width=\linewidth]{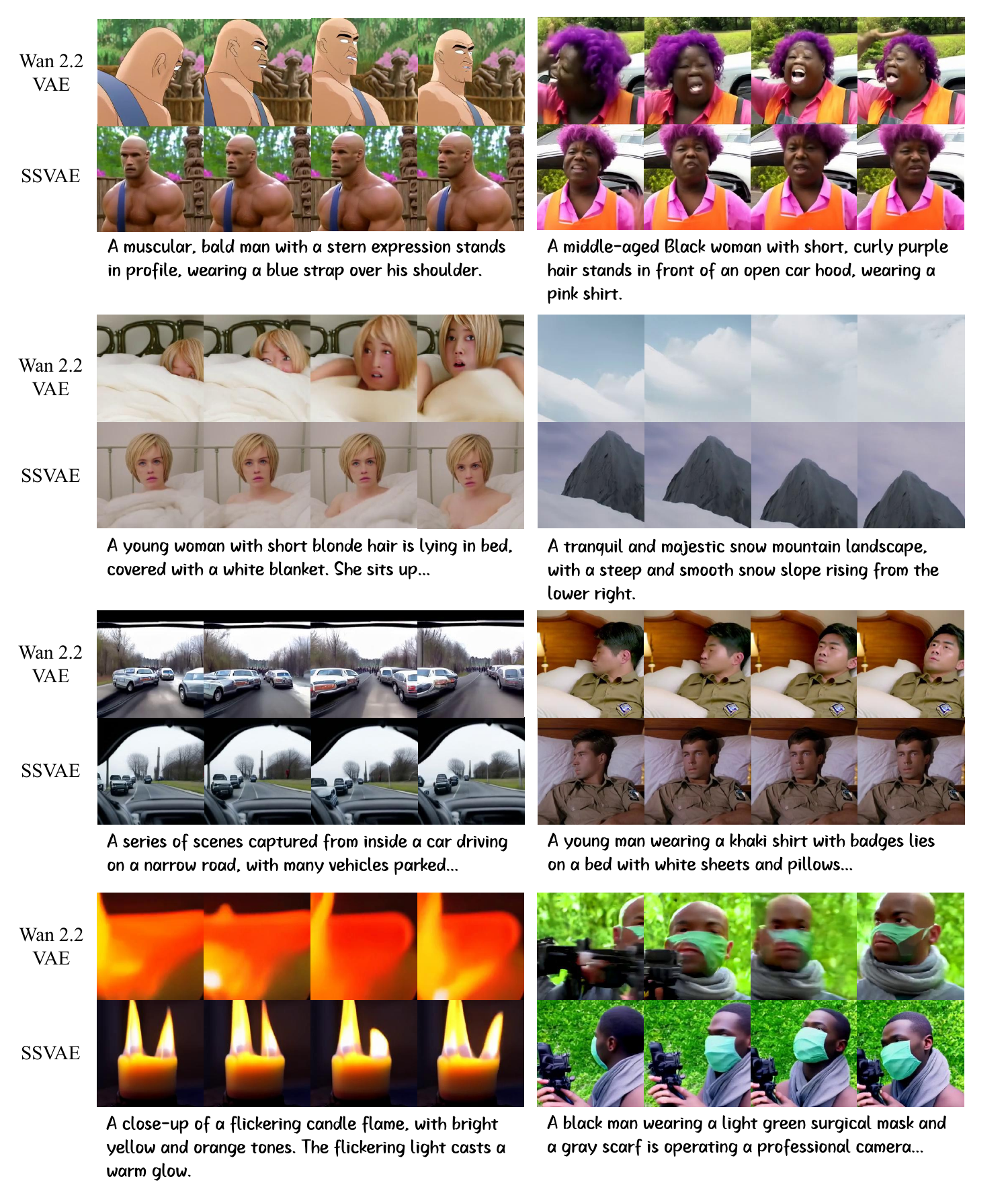}
    \caption{Qualitative comparisons with a 1.3B diffusion model at $81\times256\times256$. 
    For each prompt, we show frames sampled at equal temporal intervals for models trained with SSVAE (ours) and Wan~2.2 VAE under identical inference settings.}
    \label{fig:1.3B_5s}
\end{figure*}

\begin{figure*}[t]
    \centering
    \includegraphics[width=\linewidth]{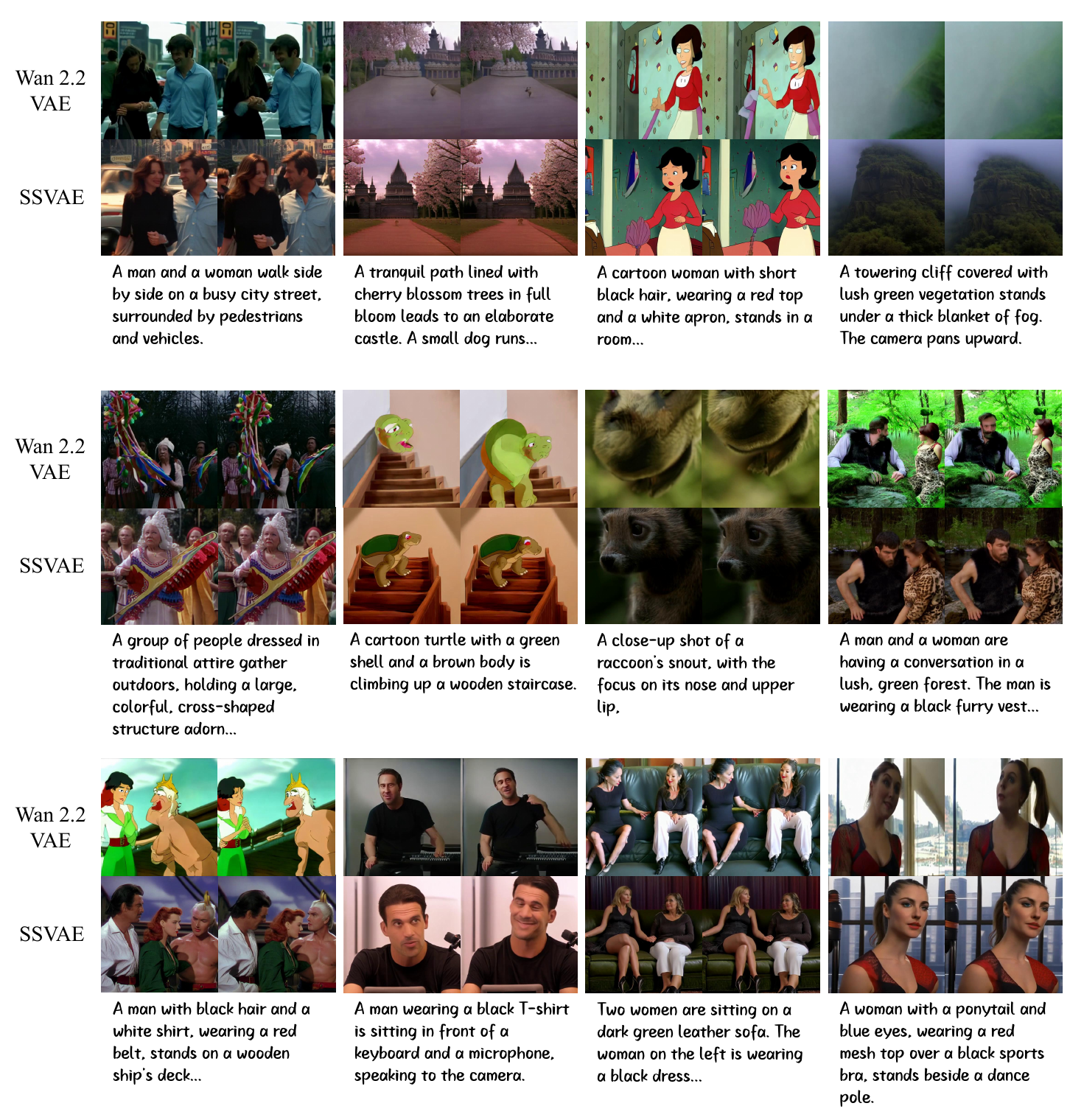}
    \caption{Side-by-side comparisons with a 4B diffusion model at $17\times512\times512$. The diffusion model trained on our VAE produces videos with reduced artifacts and better text-video alignment.}
    \label{fig:4B}
\end{figure*}

Fig.~\ref{fig:1.3B_5s} and Fig.~\ref{fig:4B} provide visualizations of two distinct generation settings. The video samples can also be found in the \texttt{gen\_videos\_1.3B\_5s} folder and \texttt{gen\_videos\_4B\_2s} of the supplementary material. With the 1.3B model at $81\times256\times256$ (Fig.~\ref{fig:1.3B_5s}), training on SSVAE reduces temporal flicker and ghosting and preserves object contours over the 5s horizon. With the 4B model at $17\times512\times512$ (Fig.~\ref{fig:4B}), side-by-side comparisons against Wan~2.2 VAE show consistent gains in motion smoothness, edge fidelity, and text rendering across diverse prompts.
Overall, the diffusion model trained on our SSVAE exhibits fewer artifacts and improved text-video alignment compared to the model trained on Wan 2.2 VAE.

\section{Computational Overhead of LCR and LMR}
\label{sec:Computation}
To further substantiate the efficiency of our regularizers, we report their wall-clock overhead measured during training in Tab.~\ref{tab:overhead}. We report the average per-step latency (in milliseconds) under the same training configuration. Overall, both regularizers introduce negligible latency: Overhead: LCR, LMR add negligible latency (1.1 \& 0.6 ms vs.\ baseline 559.8 ms).
\begin{table}[h]
  \centering
  \small
  \caption{Wall-clock overhead of LCR and LMR measured during training.}
  \label{tab:overhead}
  \begin{tabular}{lc}
    \toprule
    Method & step (ms) \\
    \midrule
    Baseline & 559.8  \\
    + LCR only & 560.9 \\
    + LMR only & 560.4 \\
    + LCR + LMR & 561.5\\
    \bottomrule
  \end{tabular}
\end{table}

\section{Prompts for UnifiedReward-Thinking}
\label{sec:Prompts}
We use UnifiedReward-Thinking to compute the UnifiedReward scores. Specifically, we uniformly sample eight frames from each video, resize them to $512\times 512$, and evaluate the generated videos from five aspects: Visual Quality, Temporal Consistency, Dynamic Degree, Text-to-Video Alignment, and Factual Consistency. The final score is a weighted average across all evaluation dimensions. The prompt used in our assessment is shown in Fig.~\ref{fig:prompt}:
\begin{figure*}[htb]
\centering
\begin{tcolorbox}[colback=gray!5!white, colframe=gray!80!black, boxrule=0.8pt, arc=3pt, left=1mm, right=1mm, top=1mm, bottom=1mm, width=\textwidth, fontupper=\ttfamily]
Suppose you are an expert in judging and evaluating the quality of AI-generated videos, please watch the frames of a given video and see the text prompt for generating the video. Then give scores from 5 different dimensions:

(1) visual quality: the quality of the video in terms of clearness, resolution, brightness, and color

(2) temporal consistency, the consistency of objects or humans in video

(3) dynamic degree, the degree of dynamic changes

(4) text-to-video alignment, the alignment between the text prompt and the video content

(5) factual consistency, the consistency of the video content with the common-sense and factual knowledge

For each evaluation dimension, provide a score between 1-10 for the video and provide a concise rationale for the score. Calculate the total score for each video by summing all dimension scores. Use a chain-of-thought process to detail your reasoning steps, and enclose all your detailed reasoning within <think> and </think> tags. Then, in the <answer> tag, output the final score in the following format: `Final score: 6'. No additional text is allowed in the <answer> section.

Example output format:

<think>

1. Visual quality: 9/10 - ...; Temporal consistency: 8/10 - ...; Dynamic degree: 7/10 - ...; Text-to-video alignment: 6/10 - ...; Factual consistency: 5/10 - ...

Total score:

9+8+7+6+5=35

</think>

<answer>Final score: 35</answer>

**Note: In the example above, scores and the final answer are placeholders meant only to demonstrate the format. Your actual evaluation should be based on the quality of the given video.

Your task is provided as follows:

Text Prompt:
\end{tcolorbox}
\caption{Prompt used for video quality evaluation.}
\label{fig:prompt}
\end{figure*}

\end{document}